\documentclass[journal]{IEEEtran}
\usepackage{amsmath,graphicx}
\usepackage{amsthm}
\newtheorem{theorem}{Theorem}
\usepackage{algorithm,algpseudocode}
\usepackage{url}
\usepackage{authblk}
\usepackage{booktabs}
\usepackage{epstopdf}
\usepackage{amssymb}
\usepackage{tikz}
\usepackage{pbox}

\def\vect{\mathbf}
\def\matr{\mathbf}
\def\mrm{\mathrm}

\DeclareMathOperator*{\argmax}{argmax}

\newcommand{\secureagg}{\mathsf{SecureAgg}}

\usepackage{cite}
\usepackage{hyperref}
\newcommand{\norm}[1]{ \left\| #1 \right \| }

\newcommand{\capepca}{\mathsf{capePCA}}
\newcommand{\dpdjica}{\mathsf{DP-djICA}}
\newcommand{\local}{\mathsf{local\ DP-ICA}}
\newcommand{\cape}{\mathsf{CAPE}}

\newcommand{\capedjica}{\mathsf{capeDJICA}}

\newcommand{\djICA}{\textsf{djICA}}

\usepackage{mathtools}
\DeclarePairedDelimiter{\ceil}{\lceil}{\rceil}
\DeclarePairedDelimiter\floor{\lfloor}{\rfloor}

\usepackage[normalem]{ulem}
\newcommand\redout{\bgroup\markoverwith{\textcolor{red}{\rule[.5ex]{2pt}{0.4pt}}}\ULon}

\newtheorem{lemma}{Lemma}
\newtheorem{Def}{Definition}
\newtheorem{Prop}{Proposition}

\newtheorem{Rem}{Remark}
\ifCLASSINFOpdf

\else

\fi

\begin{document}

\title{Improved Differentially Private Decentralized Source Separation for fMRI Data}

\author{
\IEEEauthorblockN{Hafiz Imtiaz,~
								Jafar Mohammadi,~\IEEEmembership{Member,~IEEE},
								Rogers Silva,~\IEEEmembership{Member,~IEEE},
  								Bradley Baker,~
  								Sergey M. Plis,~
  								Anand D. Sarwate,~\IEEEmembership{Senior Member,~IEEE},
  								Vince Calhoun,~\IEEEmembership{Fellow,~IEEE}}

\thanks{H.~Imtiaz is with the Department of Electrical and Electronic Engineering, Bangladesh University of Engineering and Technology, Dhaka, Bangladesh. A.D.~Sarwate is with the Department of Electrical and Computer Engineering, Rutgers University, 94 Brett Road, Piscataway, NJ 08854. J.~Mohammadi is with Nokia Bell Labs, Lorenzstrasse 10, 70435 Stuttgart, Germany. R.~Silva, B.~Baker, S.M.~Plis and V.~Calhoun are with the Tri-institutional Center for Translational Research in Neuroimaging and Data Science (TReNDS), Georgia State University, Georgia Institute of Technology, and Emory University, 55 Park Place NE, Atlanta, GA 30303.}
\thanks{Corresp. author: hafiz.imtiaz@rutgers.edu, anand.sarwate@rutgers.edu.}
\thanks{This work was supported in part by the United States NIH under award 1R01DA040487, the United States NSF under award CCF-1453432, and DARPA and SSC Pacific under contract N66001-15-C-4070. This work significantly improves upon our preliminary work~\cite{imtiaz2016}.}
}

\markboth{IEEE Transactions on Signal Processing,~Vol.~XX, No.~XX, XXXX~2020}%
{Imtiaz et al.: Improved Differentially Private Decentralized Source Separation for fMRI Data}

\maketitle

\begin{abstract}
Blind source separation algorithms such as independent component analysis (ICA) are widely used in the analysis of neuroimaging data. In order to leverage larger sample sizes, different data holders/sites may wish to collaboratively learn feature representations. However, such datasets are often privacy-sensitive, precluding centralized analyses that pool the data at a single site. In this work, we propose a differentially private algorithm for performing ICA in a decentralized data setting. Conventional approaches to decentralized differentially private algorithms may introduce too much noise due to the typically small sample sizes at each site. We propose a novel protocol that uses correlated noise to remedy this problem. We show that our algorithm outperforms existing approaches on synthetic and real neuroimaging datasets and demonstrate that it can sometimes reach the same level of utility as the corresponding non-private algorithm. This indicates that it is possible to have meaningful utility while preserving privacy.
\end{abstract}


\IEEEpeerreviewmaketitle

\section{Introduction}\label{sec:intro}
\IEEEPARstart{S}{haring} data is a major challenge facing researchers in a number of domains. In particular, human health studies often involve a modest number of individuals: privacy concerns often preclude sharing ``raw'' data with collaborators. Performing a new joint analysis across the individual data points requires access to individuals' data, so research groups often collaborate by performing meta-analyses which are limited to already-published aggregates or summaries of the data. 
For machine learning (ML) applications, each party/site may lack a sufficient number of samples to robustly estimate features on their own, but the aggregate number of samples across all sites can yield novel discoveries such as biomarkers for diseases. Sending the data samples to a central repository or aggregator can enable efficient feature learning, but privacy concerns and large communication overhead are often prohibitive when sharing ``raw'' data. Some noteworthy examples, in which many individual research groups or sites wish to collaborate, include:
\begin{itemize}
\item a medical research consortium of several healthcare centers/research labs for neuroimaging analysis~\cite{enigma, coinstac, vipar, datashield}
\item a decentralized speech processing system to learn model parameters for speaker recognition
\item a multi-party cyber-physical system for performing global state estimation from sensor signals.
\end{itemize}

\noindent Several previous works demonstrated how modern signal processing and ML algorithms can potentially reveal information about individuals present in the dataset~\cite{netflixchallenge, sweeney, leny2012}. A mathematically rigorous framework for protection against such information leaks is differential privacy~\cite{dwork2006}. Under differential privacy, the algorithm outputs are randomized in such a way that the presence or absence of any individual in the dataset does not significantly affect the computation output. In other words, differentially private (DP) algorithms offer a quantifiable plausible deniability to the data owners regarding their participation. The randomization often takes the form of noise introduced somewhere in the computation, resulting in a loss in performance or \emph{utility} of the algorithm. Privacy risk is quantified by a parameter or parameters, leading to a \emph{privacy-utility tradeoff} in DP algorithm design.

In this paper, we consider blind source separation (BSS) for neuroimaging, in which several individual research groups or sites wish to collaborate. The joint goal is to learn global statistics/features utilizing data samples from all sites and ensure formal privacy guarantees. Unfortunately, conventional approaches to using differential privacy in decentralized settings require introducing too much noise, leading to a poor tradeoff. The high dimensional nature of neuroimaging data also poses a challenge. This motivates us to develop efficient decentralized privacy-preserving algorithms that provide utility close to centralized case. The primary contribution of this paper is the development of a decentralized computation framework to reduce the amount of noise in conventional decentralized DP schemes for applications in signal processing and ML. We employ the scheme into our BSS application for neuroimaging data that guarantees differential privacy with an improved privacy-utility tradeoff.

The particular BSS algorithm we are considering is the independent component analysis (ICA), one of the most popular BSS techniques for neuroimaging studies~\cite{calhoun2013}. ICA assumes that the observed signals are mixtures of statistically independent sources and aims to decompose the mixed signals into those sources. ICA has been widely used to estimate intrinsic connectivity networks from brain imaging data (e.g. functional magnetic resonance imaging (fMRI))~\cite{calhoun2012}.
Successful application of ICA on fMRI can be attributed to both sparsity and spatial or temporal independence between the underlying sources~\cite{calhoun2013}. The goal of temporal ICA is to identify temporally independent components that represent activation of different brain regions over time~\cite{comon1994}. However, it requires the aggregate temporal dimension (of all subjects) to be at least similar to the voxel dimension~\cite{baker2018}. In most cases, the data from a single medical center may not suffice for such analysis. We focus on the recently proposed decentralized joint ICA (\djICA) algorithm, which can perform temporal ICA of fMRI data~\cite{baker2018} by allowing research groups to collaboratively learn the underlying sources.

\noindent\textbf{Our Contribution.} 
The problem with conventional decentralized DP computations is that sites have to add too much noise, resulting in degraded performance at the aggregator (see Appendix \ref{appendix:conventional-decentralized-dp} for a simple example depicting the issues associated with such an approach). In this paper, we propose a novel decentralized computation framework to improve the privacy-utility tradeoff of conventional DP schemes. We achieve this by inducing (anti) correlated noise: our method adds correlated noise to the output of each site to guarantee privacy locally and an aggregator combines these noisy outputs to produce an improved estimate due to noise cancellation. We propose a new algorithm, $\capedjica$, for $(\epsilon,\delta)$-DP decentralized joint ICA using the proposed framework. We summarize our contributions here:
\begin{itemize}
\item We propose a novel decentralized computation protocol, $\cape$, that improves upon the conventional decentralized DP schemes and achieves the same level of noise variance as the pooled data scenario in certain parameter regimes. $\cape$ can be incorporated into large-scale statistical signal processing and ML methods formulated as convex optimization problems, such as empirical risk minimization (ERM) and computation of loss functions that are separable across sites. We extend the $\cape$ scheme to incorporate asymmetric privacy requirements or sample sizes at the sites.

\item We propose a new algorithm, $\capedjica$, for $(\epsilon,\delta)$-DP decentralized joint ICA. $\capedjica$ significantly improves upon our earlier work~\cite{imtiaz2016} by taking advantage of the $\cape$ scheme. To address the multi-round nature of the $\capedjica$ algorithm and to provide a tighter characterization of privacy under composition, we provide an analysis using R\'enyi Differential Privacy (RDP)~\cite{mironov2017} and the moments accountant~\cite{abadi2016}.



\item We demonstrate the effectiveness of $\cape$ and $\capedjica$ on real and synthetic data varying privacy levels, number of samples and some other key parameters. We show that the $\capedjica$ can provide utility very close to that of the non-private algorithm~\cite{baker2018} for some parameter choices. In the regime of meaningful utility, $\capedjica$ outperforms the existing privacy-preserving algorithm~\cite{imtiaz2016}. Our improved accounting of privacy via the moments accountant, enables us to achieve such performance even for strict privacy requirements.
\end{itemize}
Note that, we showed a preliminary version of the $\cape$ protocol in~\cite{imtiaz2018}. The protocol in this paper is more robust against site dropouts and does not require a trusted third-party.

\noindent\textbf{Related Work.} 
There is a vast literature~\cite{boyd2011, molzahn2017, uribe2017, han2017, nozari2016, zhu2018} on solving optimization problems in decentralized settings, both with and without privacy constraints. In the signal processing/ML context, the most relevant ones to our current work are those using ERM and stochastic gradient descent (SGD)~\cite{Kamalika09, anand2011, AnandSGD, abadi2016, lipton2014, DPonline18, bassily2014, Katrina17, wang2017}. Additionally, several works studied decentralized DP learning for locally trained classifiers~\cite{Agarwal12,BassilyKobbi17}. One of the most common approaches for ensuring differential privacy in optimization problems is to employ randomized gradient computations~\cite{AnandSGD, bassily2014}. Other common approaches include employing  output perturbation~\cite{anand2011} and objective perturbation~\cite{anand2011, nozari2016}. A newly proposed take on output perturbation~\cite{wu2017} injects noise after model convergence, which imposes some additional constraints. In addition to optimization problems, Smith~\cite{smith2011} proposed a general approach for computing summary statistics using the \emph{sample-and-aggregate} framework and both the Laplace and Exponential mechanisms~\cite{MT07}.

One approach to handling decentralized learning with privacy constraints is federated learning~\cite{McMahan:17federated} and in particular, cross-silo learning~\cite{Kairouz:19advances}. Many approaches use multiparty computation (MPC), such as Heikkil\"{a} et al.~\cite{Mikko17}, who also studied the relationship of additive noise and sample size in a decentralized setting. In their model, $S$ data holders communicate their data to $M$ computation nodes to compute a function. Differential privacy provides different guarantees (see ~\cite{SlawomirMPC17, Kairouz2015smc} for thorough comparisons between Secure Multi-party Computation (SMC) and differential privacy) although we can use MPC protocols to implement part of our algorithm~\cite{Bonawitz17}. Other approaches to using DP in federated learning operate in different regimes, such as learning from a large number of individual data holders, or learning from silos with a large number of data points at each site. This allows for privacy amplification by subsampling~\cite{KLNRS08,Beimel13complexity,abadi2016,NIPS2018_7865} or using a trusted shuffler~\cite{BalleBN:19blanket,ErlingssonFMRTTT}. In our application involving neuroimaging analysis, such techniques do not scale as well, since sites often have a small number of samples. Our work is inspired by the seminal work of Dwork et al.~\cite{Ourdata06} that proposed distributed noise generation for preserving privacy. We employ a similar principle as Anandan and Clifton~\cite{Clifton15} to \emph{reduce} the noise added for differential privacy. Our approach seeks to leverage properties of conditional Gaussian distributions to gain some privacy amplification when learning from decentralized data, and is complementary to these other techniques.


In addition to generalized optimization methods, a number of modified ICA algorithms exist for joining various data sets~\cite{sui2009} together and performing simultaneous decomposition of data from a number of subjects and modalities~\cite{liu2007}. Note that ICA can be performed by considering voxels as variables or time points as variables, leading to temporal and spatial ICA, respectively~\cite{bordier2010, calhoun2001}. For instance, group spatial ICA (GICA) is noteworthy for performing multi-subject analysis of task- and resting-state fMRI data~\cite{allen2011, calhoun2001, calhoun2012}. It assumes that the spatial map components are similar across subjects (i.e., the overall spatial networks are stable across subjects for the experiment duration). The joint ICA (jICA)~\cite{calhoun2006} algorithm for multi-modal data fusion assumes that the mixing process is similar over a group of subjects. Group temporal ICA also assumes common spatial maps but pursues statistical independence of timecourses (activation of certain neurological regions)~\cite{baker2018}. Consequently, like jICA, the common spatial maps from temporal ICA describe a common mixing process among subjects. While very interesting, temporal ICA of fMRI is typically not investigated because of the small number of time points in each data set, which leads to unreliable estimates~\cite{baker2018}. The decentralized jICA overcomes that limitation by leveraging datasets from multiple sites.

\section{Data and Privacy Model}\label{sec:problem_formulation}
\noindent\textbf{Notation. }We denote vectors, matrices and scalars with bold lower case letters ($\vect{x}$), bold upper case letters ($\matr{X}$) and unbolded letters ($M$), respectively. We denote indices with smaller case letters and they typically run from 1 to their upper-case versions ($m \in \{1, 2, \ldots, M\} \triangleq [M]$). The $n$-th column of the matrix $\matr{X}$ is denoted as $\vect{x}_n$. We denote the Euclidean (or $\mathcal{L}_2$) norm of a vector and the spectral norm of a matrix with $\|\cdot\|_2$ and the Frobenius norm with $\|\cdot\|_F$. Finally, the density of the standard normal random variable is given by $\phi(x) = \left(1 / \sqrt{2\pi}\right) \exp \left(-x^2 / 2\right)$.

\noindent\textbf{The ICA Model.} In this paper, we consider the generative ICA model as in~\cite{baker2018, imtiaz2016}. In the centralized scenario, the independent sources $\matr{S} \in \mathbb{R}^{R\times N}$ are composed of $N$ observations from $R$ statistically independent components. We have a linear mixing process defined by a mixing matrix $\matr{A} \in \mathbb{R}^{D\times R}$ with $D \geq R$, which forms the observed data $\matr{X} \in \mathbb{R}^{D\times N}$ as a product $\matr{X} = \matr{A} \matr{S}$. Many ICA algorithms propose recovering the unmixing matrix $\matr{W}$, corresponding to the Moore-Penrose pseudo-inverse of $\matr{A}$, denoted $\matr{A}^+$, by trying to maximize independence between rows of the product $\matr{W} \matr{X}$. The maximal information transfer (infomax)~\cite{bell1995} is a popular heuristic for estimating $\matr{W}$ that maximizes an entropy functional related to $\matr{W} \matr{X}$. More specifically, the objective of Infomax ICA is: $\matr{W}^* = \argmax_\matr{W} \mathcal{H}(\mathcal{G}(\matr{W} \matr{X}))$. Here, $\mathcal{G}(\cdot)$ is the sigmoid function and is given by: $\mathcal{G}(z) = \frac{1}{1+\exp(-z)}$. Additionally, $\mathcal{H}(\vect{z})$ is the (differential) entropy of a random vector $\vect{z}$ with joint density $q$: $\mathcal{H}(\vect{z}) = - \int q(\vect{z})\log q(\vect{z}) d\vect{z}$. Note that the function $\mathcal{G}(\cdot)$ is applied element-wise for matrix-valued arguments. That is, $\mathcal{G}(\matr{Z})$ is a matrix with the same size as $\matr{Z}$ and $[\mathcal{G}(\matr{Z})]_{ij} = \mathcal{G}([\matr{Z}]_{ij})$). 

\noindent\textbf{The Decentralized Data Problem.} We consider a decentralized-data model with $S$ sites. There is a central node that acts as an aggregator. \textcolor{blue}{We assume a ``honest but curious'' threat model: all parties follow the protocol but a subset are ``curious'' and can collude (maybe with an external adversary) to learn other sites' data/function outputs.} Now, for the decentralized ICA problem, suppose each site $s$ has a collection of data matrices $\{ \matr{X}_{s,m} \in \mathbb{R}^{D\times N_m} : m = 1,\ldots, M_s \}$ each consisting of a time course of length $N_m$ time points over $D$ voxels for each of $M_s$ individuals. We assume the data samples in the local sites are disjoint and come from different individuals. Sites concatenate their local data matrices temporally to form a $D \times N_m M_s$ data matrix $\matr{X}_s \in \mathbb{R}^{D\times N_s} $, where $N_s = N_m M_s$. Let $N = \sum_{s=1}^S N_s$ be the total number of samples and $M = \sum_{s=1}^S M_s$ be the total number of individuals (across all sites). We assume a global mixing matrix $\matr{A} \in \mathbb{R}^{D\times R}$ generates the time courses in $\matr{X}_s$ from underlying sources $\matr{S}_s \in \mathbb{R}^{R\times N_s}$ at each site. This yields the following model: $\matr{X} = \left[ \matr{A} \matr{S}_1 \ldots \matr{A} \matr{S}_S \right] = \left[ \matr{X}_1 \ldots \matr{X}_S \right]\in \mathbb{R}^{D\times N}$. We want to compute the global unmixing matrix $\matr{W} \in \mathbb{R}^{R\times D}$ in the decentralized setting. Because sharing the raw data between sites is often impossible due to privacy constraints, we develop methods that guarantee differential privacy~\cite{dwork2006}. More specifically, our goal is to use DP estimates of the local gradients to compute the DP global unmixing matrix $\matr{W}$ such that it closely approximates the true global unmixing matrix. 

\subsection{Definitions}\label{sec:definitions}
In differential privacy we consider a domain $\mathbb{D}$ of databases consisting of $N$ records and define $D$ and $D'$ to be neighbors if they differ in a single record.

\begin{Def}[($\epsilon, \delta$)-Differential Privacy~\cite{dwork2006}] 
An algorithm $\mathcal{A}:  \mathbb{D} \mapsto \mathbb{T}$ provides $(\epsilon,\delta)$-differential privacy ($(\epsilon,\delta)$-DP) if $\Pr[\mathcal{A}(D) \in \mathbb{S}] \leq \exp(\epsilon) \Pr[\mathcal{A}(D') \in \mathbb{S}] + \delta$, for all measurable $\mathbb{S} \subseteq \mathbb{T}$ and all neighboring data sets $D, D' \in \mathbb{D}$.
\end{Def}
\noindent One way to interpret this is that the probability of the output of an algorithm is not changed significantly if the input database is changed by one entry. This definition is also known as the Bounded Differential Privacy (as opposed to unbounded differential privacy~\cite{dwork2006a}).
Here, $(\epsilon, \delta)$ are privacy parameters, where lower $(\epsilon, \delta)$ ensure more privacy. The parameter $\delta$ can be interpreted as the probability that the algorithm fails to provide privacy risk $\epsilon$. Note that $(\epsilon, \delta)$-differential privacy is known as the \emph{approximate} differential privacy and $\epsilon$-differential privacy ($\epsilon$-DP) is known as \emph{pure} differential privacy. In general, we denote approximate (bounded) differentially private algorithms with DP. There are several mechanisms for formulating a DP algorithm. Additive noise mechanisms such as the Gaussian or Laplace mechanisms~\cite{dwork2006, dwork2013algorithmic} and random sampling using the exponential mechanism~\cite{MT07} are among the most common ones. For additive noise mechanisms, the standard deviation of the noise is scaled to the \emph{sensitivity} of the computation.

\begin{Def}[$\mathcal{L}_2$-sensitivity~\cite{dwork2006}]
The $\mathcal{L}_2$-sensitivity of a vector-valued function $f(\mathbb{D})$ is $\Delta := \max_{ \mathbb{D}, \mathbb{D'} } \|f(\mathbb{D})-f(\mathbb{D'})\|_2$,  where $\mathbb{D}$ and $\mathbb{D'}$ are neighboring datasets. 
\end{Def}

\begin{Def}[Gaussian Mechanism~\cite{dwork2013algorithmic}]\label{def:GaussMech}
Let $f: \mathbb{D} \mapsto \mathbb{R}^D$ be an arbitrary function with $\mathcal{L}_2$-sensitivity $\Delta$. The Gaussian Mechanism with parameter $\tau$ adds noise scaled to $\mathcal{N}(0, \tau^2)$ to each of the $D$ entries of the output and satisfies $(\epsilon, \delta)$ differential privacy for $\epsilon \in (0,1)$ if $\tau \geq \frac{\Delta}{\epsilon}\sqrt{2\log\frac{1.25}{\delta}}$.
\end{Def}
\noindent Note that, for any given $(\epsilon, \delta)$ pair, we can calculate a noise variance $\tau^2$ such that addition of a noise term drawn from $\mathcal{N}(0, \tau^2)$ guarantees $(\epsilon, \delta)$-differential privacy. There are infinitely many $(\epsilon, \delta)$ pairs that yield the same $\tau^2$. Therefore, we parameterize our methods using $\tau^2$~\cite{imtiaz2018} in this paper.

\begin{Def}[R\'enyi Differential Privacy~\cite{mironov2017}]
A randomized mechanism $\mathcal{A}:  \mathbb{D} \mapsto \mathbb{T}$ is $(\alpha, \epsilon_r)$-R\'enyi differentially private if, for any adjacent $D, D' \in \mathbb{D}$, the following holds: $D_\alpha\left(\mathcal{A}(D)\|\mathcal{A}(D')\right) \leq \epsilon_r$. Here, $D_\alpha\left(P(x)\|Q(x)\right) = \frac{1}{\alpha-1} \log \mathbb{E}_{x \sim Q}\left(\frac{P(x)}{Q(x)}\right)^\alpha$ and $P(x)$ and $Q(x)$ are probability density functions defined on $\mathbb{T}$. 
\end{Def}
\noindent We use RDP to perform an analysis of the $\capedjica$ and convert those guarantees into an $(\epsilon,\delta)$-DP guarantee. Conventional privacy analysis of multi-shot algorithms tend to exaggerate the total privacy loss~\cite{abadi2016, mironov2017}. RDP offers a much simpler composition rule that is shown to be tight~\cite{mironov2017}. Our solution requires the $\secureagg$ scheme~\cite{Bonawitz17} as a primitive: details about this protocol are in Appendix \ref{appendix:secureagg}. Briefly, it can securely compute sums of vectors in an honest-but-curious setup, has a constant number of rounds and a low communication overhead.

\section{Correlation Assisted Private Estimation ($\cape$)}\label{sec:cape}
In this section, we describe our $\cape$ scheme which can benefit a broad class of function computations, including empirical average loss functions used in ML. We first discuss the $\cape$ scheme and analyze the privacy and utility in the symmetric setting (equal sample sizes and privacy requirements at the sites). Recognizing practical applications involving unequal sample size/privacy requirements at sites, we then extend the $\cape$ scheme to incorporate such scenarios. 

\begin{algorithm}[t] 
	\caption{Generate zero-sum noise \label{alg:zero-sum-noise-generation}}
	\begin{algorithmic}[1]
    \Require 
    Local noise variances $\{\tau_s^2\}$; security parameter $\lambda$, threshold value $t$
    \State Each site generate $\hat{e}_s \sim \mathcal{N}(0, \tau_s^2)$
    \State Aggregator computes $\sum_{s=1}^S \hat{e}_s$ according to $\secureagg(\lambda, t)$~\cite{Bonawitz17}
    \State Aggregator broadcasts $\sum_{s=1}^S \hat{e}_s$ to all sites $s \in [S]$
    \State Each site computes $e_s = \hat{e}_s - \frac{1}{S}\sum_{s'=1}^S \hat{e}_{s'}$\\
    \Return $e_s$
    \end{algorithmic}
\end{algorithm}

\begin{algorithm}[t] 
	\caption{Correlation Assisted Private Estimation ($\cape$)\label{alg:cape}}
	\begin{algorithmic}[1]
    \Require Data samples $\{\vect{x}_s\}$, 
    local noise variances $\{\tau_s^2\}$
    \For{$s = 1,\ \ldots,\ S$} \Comment{at each site}
        \State Generate $e_s$ according to Algorithm \ref{alg:zero-sum-noise-generation}
        \State Generate $g_s \sim \mathcal{N}(0, \tau_g^2)$ with $\tau_g^2 = \frac{\tau_s^2}{S}$
        \State Compute and send $\hat{a}_s \gets f(\vect{x}_s) + e_s + g_s$ \label{alg:cape:step_as_hat}
    \EndFor
    \State Compute $a_\mrm{cape} \gets \frac{1}{S} \sum_{s=1}^S \hat{a}_s$ \label{alg:cape:step_a_ag} \Comment{at the aggregator}\\
    \Return $a_\mrm{cape}$
    \end{algorithmic}
\end{algorithm}

\noindent\textbf{Trust/Collusion Model. }In our proposed $\cape$ scheme, we assume that all of the $S$ sites and the central node follow the protocol honestly. However, up to $S_C = \ceil*{\frac{S}{3}} - 1$ sites can collude with an adversary to learn about some site's data/function output. The central node is also honest-but-curious (and therefore, can collude with an adversary). An adversary can observe the outputs from each site, as well as the output from the aggregator. Additionally, the adversary can know everything about the colluding sites (including their private data). We denote the number of non-colluding sites with $S_H$ such that $S = S_C + S_H$. Without loss of generality, we designate the non-colluding sites with $\{1, \ldots, S_H\}$.

\noindent\textbf{Correlated Noise. }We design the noise generation procedure such that: i) we can ensure $(\epsilon, \delta)$ differential privacy of the algorithm output from each site and ii) achieve the noise level of the pooled data scenario in the final output from the aggregator. We achieve that by employing a correlated noise addition scheme.

Consider estimating the mean $f(\vect{x}) = \frac{1}{N} \sum_{n = 1}^N x_n$ of $N$ scalars $\vect{x} = [ x_1,\ldots, x_{N-1},\ x_N]^{\top}$ with each $x_i \in [0,1]$. The sensitivity of the function $f(\vect{x})$ is $\frac{1}{N}$. Assume that the $N$ samples are equally distributed among $S$ sites. That is, each site $s \in \{1,\ldots,S\}$ holds a disjoint dataset $\vect{x}_s$ of $N_s = N/S$ samples. An aggregator wishes to estimate and publish the mean of all the samples (see Appendix \ref{appendix:conventional-decentralized-dp} for the issues associated with conventional decentralized DP approach towards this problem). In our proposed scheme, we intend to release (and send to the aggregator) $\hat{a}_s = f(\vect{x}_s) + e_s + g_s$ from each site $s$, where $e_s$ and $g_s$ are two noise terms. The variances of $e_s$ and $g_s$ are chosen to ensure that the noise $e_s + g_s$ is sufficient to guarantee $(\epsilon, \delta)$-differential privacy to $f(\vect{x}_s)$.
Each site generates the noise $g_s \sim \mathcal{N}(0, \tau_g^2)$ locally and the noise $e_s \sim \mathcal{N}(0, \tau_e^2)$ jointly with all other sites such that $\sum_{s=1}^S e_s = 0$. We employ the recently proposed secure aggregation protocol $(\secureagg)$ by Bonawitz et al.~\cite{Bonawitz17} to generate $e_s$ that ensures $\sum_{s=1}^S e_s = 0$. The $\secureagg$ protocol utilizes Shamir's $t$-out-of-$n$ secret sharing~\cite{shamir1979} and is communication-efficient (see Section \ref{appendix:secureagg}). 

\noindent\textbf{Detailed Description of $\cape$ Protocol. }In our proposed scheme, each site $s \in [S]$ generates a noise term $\hat{e}_s \sim \mathcal{N}(0, \tau_s^2)$ independently. The aggregator computes $\sum_{s=1}^S \hat{e}_s$ according to the $\secureagg$ protocol and broadcasts it to all the sites. Each site then sets $e_s = \hat{e}_s - \frac{1}{S}\sum_{s'=1}^S \hat{e}_{s'}$ to achieve $\sum_{s=1}^S e_s = 0$. We show the complete noise generation procedure in Algorithm \ref{alg:zero-sum-noise-generation}. The security parameter $\lambda$ can be chosen according to the dataset in consideration (see Section 3.1 in~\cite{Bonawitz17}). Additionally, the threshold $t$ can be chosen to satisfy $t \geq \floor{\frac{2S}{3}} + 1$. Note that, the original $\secureagg$ protocol is intended for computing sum of $D$-dimensional vectors in a finite field $\mathbb{Z}^D_\lambda$. However, we need to perform the summation of Gaussian random variables over $\mathbb{R}$ or $\mathbb{R}^D$. To accomplish this, each site can employ a mapping $\mrm{map}: \mathbb{R} \mapsto \mathbb{Z}_\lambda$ that performs a  stochastic quantization~\cite{salman2019} for large-enough $\lambda$. The aggregator can compute the sum in the finite field according to $\secureagg$ and then invoke a reverse mapping $\mrm{remap}: \mathbb{Z}_\lambda \mapsto \mathbb{R}$ before broadcasting $\sum_{s=1}^S \hat{e}_s$ to the sites. Algorithm \ref{alg:zero-sum-noise-generation} can be readily extended to generate array-valued zero-sum noise terms. We observe that the variance of $e_s$ is given by $\tau_e^2 = \mathbb{E}\left[\left(\hat{e}_s - \frac{1}{S}\sum_{s'=1}^S \hat{e}_{s'}\right)^2\right] = \left(1-\frac{1}{S}\right)\tau^2_s$. Additionally, we choose $\tau_g^2 = \frac{\tau^2_s}{S}$. Each site then generates the noise $g_s \sim \mathcal{N}(0, \tau_g^2)$ independently and sends $\hat{a}_s = f(\vect{x}_s) + e_s + g_s$ to the aggregator. Note that neither of the terms $e_s$ and $g_s$ has large enough variance to provide $(\epsilon, \delta)$-DP guarantee to $f(\vect{x}_s)$. However, we chose the variances of $e_s$ and $g_s$ to ensure that the $e_s + g_s$ is sufficient to ensure a DP guarantee to $f(\vect{x}_s)$ at site $s$. The chosen variance of $g_s$ also ensures that the output from the aggregator would have the same noise variance as the DP pooled-data scenario. To see this, observe that we compute the following at the aggregator (in Step \ref{alg:cape:step_a_ag} of Algorithm \ref{alg:cape}): $a_\mrm{cape} = \frac{1}{S} \sum_{s=1}^S \hat{a}_s = \frac{1}{S} \sum_{s=1}^S f(\vect{x}_s) + \frac{1}{S} \sum_{s=1}^S g_s$, where we used $\sum_s e_s = 0$. The variance of the estimator $a_\mrm{cape}$ is $\tau^2_\mrm{cape} = S \cdot \frac{\tau_g^2}{S^2} = \tau^2_\mrm{pool}$, which is exactly the same as if all the data were present at the aggregator. This claim is formalized in Lemma~\ref{lemma:cape}. We show the complete algorithm in Algorithm \ref{alg:cape}. The privacy of Algorithm \ref{alg:cape} is given by Theorem \ref{thm:cape}. The communication cost of the scheme is shown in Appendix~\ref{appendix:cape_comm}.

\begin{theorem}[Privacy of $\cape$ Algorithm (Algorithm \ref{alg:cape})]\label{thm:cape}
Consider Algorithm \ref{alg:cape} in the decentralized data setting of Section \ref{sec:intro} with $N_s = \frac{N}{S}$ and $\tau_s^2 = \tau^2$ for all sites $s \in [S]$. Suppose that at most $S_C = \ceil*{\frac{S}{3}} - 1$ sites can collude after execution. Then Algorithm \ref{alg:cape} guarantees $(\epsilon, \delta)$-differential privacy for each site, where $(\epsilon,\delta)$ satisfy the relation $\delta = 2\frac{\sigma_z}{\epsilon - \mu_z}\phi\left(\frac{\epsilon - \mu_z}{\sigma_z}\right)$, $\epsilon \in (0,1)$ and $(\mu_z, \sigma_z)$ are given by
\begin{align}
    \mu_z  &= \frac{S^3}{2\tau^2 N^2 (1+S)} \left(\frac{S - S_C + 2}{S - S_C} + \frac{\frac{9}{S - S_C}S_C^2}{S(1+S) - 3S_C^2}\right), 
    	\label{eq:CAPEmu} \\
    \sigma_z^2 &= 2 \mu_z. \label{eq:CAPEsigma} 
\end{align}
\end{theorem}

\begin{Rem}
Theorem~\ref{thm:cape} is stated for the symmetric setting: $N_s = \frac{N}{S}$ and $\tau_s^2 = \tau^2\ \forall s \in [S]$. As with many algorithms using the approximate differential privacy, the guarantee holds for a range of $(\epsilon,\delta)$ pairs subject to a tradeoff constraint between $\epsilon$ and $\delta$, as in the simple case in Definition~\ref{def:GaussMech}.
\end{Rem}

\begin{proof}
As mentioned before, we identify the $S_H$ non-colluding sites with $s \in \{1, \ldots, S_H\} \triangleq \mathbb{S}_H$ and the $S_C$ colluding sites with $s \in \{S_H + 1, \ldots, S\} \triangleq \mathbb{S}_C$. The adversary can observe the outputs from each site (including the aggregator). Additionally, the colluding sites can share their private data and the noise terms, $\hat{e}_s$ and $g_s$ for $s \in \mathbb{S}_C$, with the adversary. For simplicity, we assume that all sites have equal number of samples (i.e., $N_s = \frac{N}{S}$) and $\tau_s^2 = \tau^2$.

To infer the private data of the sites $s \in \mathbb{S}_H$, the adversary can observe $\hat{\vect{a}} = \left[\hat{a}_1, \ldots, \hat{a}_{S_H}\right]^\top \in \mathbb{R}^{S_H}$ and $\hat{e} = \sum_{s \in \mathbb{S}_H} \hat{e}_s$. Note that the adversary can learn the partial sum $\hat{e}$ because they can get the sum $\sum_s \hat{e}_s$ from the aggregator and the noise terms $\{\hat{e}_{S_H + 1}, \ldots, \hat{e}_S\}$ from the colluding sites. Therefore, the adversary observes the vector $\vect{y} = \left[\hat{\vect{a}}^\top, \hat{e}\right]^\top \in \mathbb{R}^{S_H + 1}$ to make inference about the non-colluding sites. To prove differential privacy guarantee, we must show that $\left|\log\frac{g(\vect{y} | \vect{a})}{g(\vect{y} | \vect{a}')}\right| \leq \epsilon$ holds with probability (over the randomness of the mechanism) at least $1-\delta$. Here, $\vect{a} = \left[f(\vect{x}_1), \ldots, f(\vect{x}_{S_H})\right]^\top$ and $g(\cdot | \vect{a})$ and $g(\cdot | \vect{a}')$ are the probability density functions of $\vect{y}$ under $\vect{a}$ and $\vect{a}'$, respectively. The vectors $\vect{a}$ and $\vect{a}'$ differ in only one coordinate (neighboring). Without loss of generality, we assume that $\vect{a}$ and $\vect{a}'$ differ in the first coordinate. We note that the maximum difference is $\frac{1}{N_s}$ as the sensitivity of the function $f(\vect{x_s})$ is $\frac{1}{N_s}$. Recall that we release $\hat{a}_s = f(\vect{x}_s) + e_s + g_s$ from each site. We observe $\forall s \in [S]$: $\mathbb{E}(\hat{a}_s) = f(\vect{x}_s),\ \mrm{var}(\hat{a}_s) = \tau^2$. Additionally, $\forall s_1 \neq s_2 \in [S]$, we have: $\mathbb{E}(\hat{a}_{s_1} \hat{a}_{s_2}) = f(\vect{x}_{s_1}) f(\vect{x}_{s_2}) - \frac{\tau^2}{S}$. That is, the random variable $\hat{\vect{a}}$ is $\mathcal{N}(\vect{a}, \Sigma_{\hat{\vect{a}}})$, where $\Sigma_{\hat{\vect{a}}} = (1+\frac{1}{S})\tau^2 \matr{I} - \vect{1}\vect{1}^\top \frac{\tau^2}{S} \in \mathbb{R}^{S_H \times S_H}$ and $\vect{1}$ is a vector of all ones. Without loss of generality, we can assume~\cite{dwork2013algorithmic} that $\vect{a} = \vect{0}$ and $\vect{a}' = \vect{a} - \vect{v}$, where $\vect{v} = \left[\frac{1}{N_s},0, \ldots, 0\right]^\top$. Additionally, the random variable $\hat{e}$ is $\mathcal{N}(0, \tau^2_{\hat{e}})$, where $\tau^2_{\hat{e}} = S_H \tau^2$. Therefore, $g(\vect{y} | \vect{a})$ is the density of $\mathcal{N}(\vect{0}, \Sigma)$, where $\Sigma =
\begin{bmatrix}
    \Sigma_{\hat{\vect{a}}} 										& \Sigma_{\hat{\vect{a}}\hat{e}} \\
    \Sigma_{\hat{\vect{a}}\hat{e}}^\top			& \tau^2_{\hat{e}}
\end{bmatrix} \in \mathbb{R}^{(S_H + 1) \times (S_H + 1)}$. With some simple algebra, we can find the expression for $\Sigma_{\hat{\vect{a}}\hat{e}}$: $\Sigma_{\hat{\vect{a}}\hat{e}} = \left(1-\frac{S_H}{S}\right)\tau^2 \vect{1} \in \mathbb{R}^{S_H}$. If we denote $\tilde{\vect{v}} = \left[\vect{v}^\top, 0\right]^\top \in \mathbb{R}^{S_H+1}$ then we observe
\begin{align*}
    \left|\log\frac{g(\vect{y} | \vect{a})}{g(\vect{y} | \vect{a}')}\right| &= \left|-\frac{1}{2}\left( \vect{y}^\top\Sigma^{-1}\vect{y} - \left(\vect{y} + \tilde{\vect{v}}\right)^\top\Sigma^{-1}\left(\vect{y} + \tilde{\vect{v}}\right)\right)\right| \\
    &= \left|\frac{1}{2}\left( 2\vect{y}^\top\Sigma^{-1}\tilde{\vect{v}} + \tilde{\vect{v}}^\top\Sigma^{-1}\tilde{\vect{v}}\right)\right| \\
    &= \left|\vect{y}^\top\Sigma^{-1}\tilde{\vect{v}} + \frac{1}{2}\tilde{\vect{v}}^\top\Sigma^{-1}\tilde{\vect{v}}\right| = |z|,
\end{align*}
where $z = \vect{y}^\top\Sigma^{-1}\tilde{\vect{v}} + \frac{1}{2}\tilde{\vect{v}}^\top\Sigma^{-1}\tilde{\vect{v}}$. Using the matrix inversion lemma for block matrices~\cite[Section 0.7.3]{horn2012} and some algebra, we have
\[
\Sigma^{-1} = 
\begin{bmatrix}
    \Sigma_{\hat{\vect{a}}}^{-1} + \frac{1}{K} \Sigma_{\hat{\vect{a}}}^{-1} \Sigma_{\hat{\vect{a}}\hat{e}}	 \Sigma_{\hat{\vect{a}}\hat{e}}^\top\Sigma_{\hat{\vect{a}}}^{-1}								& -\frac{1}{K} \Sigma_{\hat{\vect{a}}}^{-1} \Sigma_{\hat{\vect{a}}\hat{e}} \\
    -\frac{1}{K}\Sigma_{\hat{\vect{a}}\hat{e}}^\top\Sigma_{\hat{\vect{a}}}^{-1}			& \frac{1}{K}
\end{bmatrix},
\]
where $\Sigma_{\hat{\vect{a}}}^{-1} = \frac{S}{(1+S)\tau^2}\left(\matr{I} + \frac{2}{S_H}\vect{1}\vect{1}^\top\right)$ and $K = \tau^2_{\hat{e}} - \Sigma_{\hat{\vect{a}}\hat{e}}^\top\Sigma_{\hat{\vect{a}}}^{-1}\Sigma_{\hat{\vect{a}}\hat{e}}$. Note that $z$ is a Gaussian random variable $\mathcal{N}(\mu_z, \sigma_z^2)$ with parameters $\mu_z 	= \frac{1}{2}\tilde{\vect{v}}^\top\Sigma^{-1}\tilde{\vect{v}}$ and $\sigma_z^2 = \tilde{\vect{v}}^\top\Sigma^{-1}\tilde{\vect{v}}$ given by \eqref{eq:CAPEmu} and \eqref{eq:CAPEsigma}, respectively. Now, we observe
\begin{align*}
    \Pr\left[\left|\log\frac{g(\vect{y} | \vect{a})}{g(\vect{y} | \vect{a}')}\right| \leq \epsilon\right] &= \Pr\left[\left|z\right| \leq \epsilon\right] = 1 - 2 \Pr\left[z > \epsilon\right] \\
    &= 1 - 2Q\left(\frac{\epsilon - \mu_z}{\sigma_z}\right)\\
    &> 1 - 2\frac{\sigma_z}{\epsilon - \mu_z}\phi\left(\frac{\epsilon - \mu_z}{\sigma_z}\right),
\end{align*}
where $Q(\cdot)$ is the Q-function~\cite{qfunc} and $\phi(\cdot)$ is the density for standard Normal random variable. The last inequality follows from the bound $Q(x) < \frac{\phi(x)}{x}$~\cite{qfunc}. Therefore, the proposed $\cape$ ensures $(\epsilon, \delta)$-DP with $\delta = 2\frac{\sigma_z}{\epsilon - \mu_z}\phi\left(\frac{\epsilon - \mu_z}{\sigma_z}\right)$ for each site, assuming that the number of colluding sites is at-most $\ceil*{\frac{S}{3}} - 1$. As the local datasets are disjoint and differential privacy is invariant under post processing, the release of $a_\mrm{cape}$ also satisfies $(\epsilon, \delta)$ differential privacy. The scheme fails to provide formal privacy if $S_C \leq \ceil*{\frac{S}{3}} - 1$ is not satisfied.
\end{proof}

\begin{Rem}
We use the $\secureagg$ protocol~\cite{Bonawitz17} to generate the zero-sum noise terms by mapping floating point numbers to a finite field. Such mappings are shown to be vulnerable to certain attacks~\cite{mironov2012}. However, the floating point implementation issues are out of scope for this paper. We refer the reader to the work of Balcer and Vadhan~\cite{balcer2018} for possible remedies. 
\end{Rem}

\subsection{Utility Analysis}\label{sec:cape_utility}
The goal is to ensure $(\epsilon, \delta)$ differential privacy for each site and achieve $\tau_\mrm{cape}^2 = \tau_\mrm{pool}^2$ at the aggregator (see Lemma~\ref{lemma:cape}). The $\cape$ protocol guarantees $(\epsilon, \delta)$ differential privacy with $\delta = 2\frac{\sigma_z}{\epsilon - \mu_z}\phi\left(\frac{\epsilon - \mu_z}{\sigma_z}\right)$. We claim that this $\delta$ guarantee is much better than the $\delta$ guarantee in the conventional decentralized DP scheme. We empirically validate this claim by comparing $\delta$ with $\delta_\mrm{conv}$ and $\delta_\mrm{pool}$ in Appendix~\ref{appendix:eff_delta}. Here, $\delta_\mrm{conv}$ and $\delta_\mrm{pool}$ are the smallest $\delta$ guarantees we can afford in the conventional decentralized DP scheme and the pooled-data scenario to achieve the same noise variance as the pooled-data scenario for a given $\epsilon$. Additionally, we empirically compare $\delta$, $\delta_\mrm{conv}$ and $\delta_\mrm{pool}$ for weaker collusion assumptions in Appendix~\ref{appendix:eff_delta}. In both cases, we observe that $\delta$ is always smaller than $\delta_\mrm{conv}$ and smaller than $\delta_\mrm{pool}$ for some $\tau$ values. That is, for achieving the same noise level at the aggregator as the pooled-data scenario, we are ensuring a much better privacy guarantee by employing the $\cape$ scheme over the conventional approach. 

\begin{lemma}\label{lemma:cape}
Consider the symmetric setting: $N_s = \frac{N}{S}$ and $\tau_s^2 = \tau^2$ for all sites $s \in [S]$. Let the variances of the noise terms $e_s$ and $g_s$ (Step \ref{alg:cape:step_as_hat} of Algorithm \ref{alg:cape}) be $\tau_e^2 = \left(1-\frac{1}{S}\right)\tau^2$ and $\tau_g^2 = \frac{\tau^2}{S}$, respectively. If we denote the variance of the additive noise (for preserving privacy) in the pooled data scenario by $\tau_\mrm{pool}^2$ and the variance of the estimator $a_\mrm{cape}$ (Step \ref{alg:cape:step_a_ag} of Algorithm \ref{alg:cape}) by $\tau_\mrm{cape}^2$ then Algorithm \ref{alg:cape} achieves the same noise variance as the pooled-data scenario (i.e., $\tau_\mrm{pool}^2 = \tau_\mrm{cape}^2$). 
\end{lemma}

\begin{proof}
The proof is given in Appendix \ref{appendix:cape_lemma}.
\end{proof}

\begin{Prop}\label{prop:gain}(Performance improvement) If the local noise variances are $\{\tau_s^2\}$ for $s \in [S]$ then the $\cape$ scheme provides a reduction $G = \frac{\tau_\mrm{conv}^2}{\tau_\mrm{cape}^2} = S$ in noise variance over conventional decentralized DP scheme in the symmetric setting ($N_s = \frac{N}{S}$ and $\tau_s^2 = \tau^2\ \forall s \in [S]$), where  $\tau_\mrm{conv}^2$ and $\tau_\mrm{cape}^2$ are the noise variances of the final estimate at the aggregator in the conventional scheme and the $\cape$ scheme, respectively.
\end{Prop}

\begin{proof}
The proof is given in Appendix \ref{appendix:perf_gain}.
\end{proof}

\begin{Rem}[Unequal Sample Sizes at Sites]\label{rem:unequal-sample-size}
The $\cape$ algorithm achieves the same noise variance as the pooled-data scenario (i.e., $\tau_\mrm{cape}^2 = \tau_\mrm{pool}^2$) in the symmetric setting: $N_s = \frac{N}{S}$ and $\tau_s^2 = \tau^2$ $\forall$ $s \in [S]$. In general, the ratio $H(\vect{n}) = \frac{\tau_\mrm{cape}^2}{\tau_\mrm{pool}^2}$, where $\vect{n} \triangleq [N_1,\ N_2,\ \ldots, N_S]$, is a function of the sample sizes in the sites. We observe: $H(\vect{n}) = \frac{N^2}{S^3} \sum_{s=1}^S \frac{1}{N_s^2}$. As $H(\vect{n})$ is a Schur-convex function, it can be shown using majorization theory~\cite{majorization} that $1 \leq H(\vect{n}) \leq \frac{N^2}{S^3}\left(\frac{1}{\left(N - S + 1\right)^2} + S - 1\right)$, where the minimum is achieved for the symmetric setting (i.e., $N_s = \frac{N}{S}$). That is, $\cape$ achieves the smallest noise variance at the aggregator in the symmetric setting.
\end{Rem}

\begin{Rem}[Site Dropouts]
Even in the case of site drop-out, the $\cape$ scheme achieves $\sum_s e_s = 0$, as long as the number of active sites is above some threshold (see Bonawitz et al.~\cite{Bonawitz17} for details). Therefore, the performance improvement of $\cape$~(Proposition~\ref{prop:gain}) remains the same irrespective of the number of dropped-out sites, as long as the number of colluding sites does not exceed $S_C = \ceil*{\frac{S}{3}} - 1$.
\end{Rem}

\subsection{Applicability of $\cape$}\label{sec:cape_scope}
$\cape$ is motivated by scientific research collaborations that are common in human health research. Privacy regulations prevent sites from sharing the local raw data. 
Joint learning across datasets can yield discoveries that are impossible to obtain from a single site. $\cape$ can benefit computations with sensitivities satisfying some conditions (see Proposition \ref{prop:low_sensitivity}). In addition to simple averages, many functions of interest have sensitivities that satisfy such conditions. Examples include the empirical average loss functions used in ML and deep neural networks. Moreover, we can use the Stone-Weierstrass theorem~\cite{rudin1976} to approximate a loss function in decentralized setting applying $\cape$ and then use off-the-shelf optimizers. 
Additional applications include optimization algorithms, $k$-means clustering and estimating probability distributions.

\begin{Prop}\label{prop:low_sensitivity}
Consider a decentralized setting with $S > 1$ sites in which site $s \in [S]$ has a dataset $D_s$ of $N_s$ samples and $\sum_{s=1}^S N_s = N$. Suppose the sites are employing the $\cape$ scheme to compute a function $f(D)$ with $\mathcal{L}_2$ sensitivity $\Delta(N)$. Denote $\vect{n} = [N_1,\ N_2,\ \ldots, N_S]$ and observe the ratio $H(\vect{n}) = \frac{\tau_\mrm{cape}^2}{\tau_\mrm{pool}^2} = \frac{\sum_{s=1}^S \Delta^2(N_s)}{S^3 \Delta^2(N)}$. Then the $\cape$ protocol achieves $H(\vect{n}) = 1$, if i) $\Delta\left(\frac{N}{S}\right) = S \Delta (N)$ for convex $\Delta (N)$; and ii) $S^3 \Delta^2(N) = \sum_{s=1}^S \Delta^2(N_s)$ for general $\Delta (N)$.
\end{Prop}

\begin{proof}
The proof is given in Appendix~\ref{appendix:low_sensitivity}
\end{proof}

\subsection{Extension of $\cape$: Unequal Sample Sizes/Privacy Requirements at Sites}\label{sec:unequal-sample-size}
Recall that $\cape$ achieves the smallest noise variance at the aggregator in the symmetric setting (see Remark~\ref{rem:unequal-sample-size}). However, in practice, there would be scenarios where different sites have different privacy requirements and/or sample sizes. Additionally, sites may want the aggregator to use different weights for different sites (possibly according to the quality of the output from a site). A scheme for doing so is shown in~\cite{imtiaz2018}. In this work, we propose a generalization of the $\cape$ scheme that can be applied in asymmetric settings. Note that the challenge of this analysis is due to the correlated noise terms with different variances (or sample sizes).

Let us assume that site $s$ requires local noise standard deviation $\tau_s$. To initiate the $\cape$ protocol, each site will generate $\hat{e}_s ~\sim \mathcal{N}(0, \sigma_s^2)$ and $g_s ~\sim \mathcal{N}(0, \tau_{gs}^2)$. The aggregator intends to compute a weighted average of each site's data/output with weights selected according to some quality measure. For example, if the aggregator knows that a particular site is suffering from more noisy observations than other sites, it can choose to give the output from that site less weight while combining the site results. Let us denote the weights by $\{\mu_s\}$ such that $\sum_{s=1}^S \mu_s = 1$ and $\mu_s \geq 0$. First, the aggregator computes $\sum_{s = 1}^S \mu_s \hat{e}_s$ using the $\secureagg$ protocol~\cite{Bonawitz17} and broadcasts it to all sites. Each site then sets $e_s = \hat{e}_s - \frac{1}{\mu_s S} \sum_{i=1}^S\mu_i\hat{e}_i$, to achieve $\sum_{s=1}^S \mu_s e_s = 0$ and releases $\hat{a}_s = f(\vect{x}_s) + e_s + g_s$. Now, the aggregator computes $a_\mrm{cape} = \sum_{s=1}^S \mu_s \hat{a}_s = \sum_{s=1}^S \mu_s f(\vect{x}_s) + \sum_{s=1}^S \mu_s g_s$, where we used $\sum_{s=1}^S \mu_s e_s = 0$. In order to achieve the same utility as the pooled data scenario (i.e. $\tau_\mrm{pool}^2 = \tau_\mrm{cape}^2$), we need $\text{Var}\left[\sum_{s=1}^S \mu_s g_s\right] = \tau_\mrm{pool}^2 \implies \sum_{s=1}^S \mu_s^2 \tau_{gs}^2 = \tau_\mrm{pool}^2$. Additionally, for guaranteeing the same local noise variance as conventional approach, we need $\tau_{es}^2 + \tau_{gs}^2 = \tau_s^2$, where $\tau_{es}^2$ is the variance of $e_s$ and is a function of $\sigma_s^2$. With these constraints, we can formulate a feasibility problem to solve for the unknown noise variances $\{\sigma_s^2, \tau_{gs}^2\}$ as
\vspace{-0.1in}
\begin{align*}
	\underset{}{\text{minimize}} \quad 0 \quad \text{subject to} 	\quad \tau_{es}^2 + \tau_{gs}^2 = \tau_s^2; \sum_{s=1}^S \mu_s^2 \tau_{gs}^2 = \tau_\mrm{pool}^2
\end{align*}
for all $s\in [S]$, where $\{\mu_s\}$, $\tau_\mrm{pool}$ and $\{\tau_s\}$ are known to the aggregator. For this problem, multiple solutions are possible. We present one solution here along with the privacy analysis.

\noindent\textbf{Solution. }We observe that the variance $\tau_{es}^2$ of the zero-mean random variable $e_s = \hat{e}_s - \frac{1}{\mu_s S} \sum_{i=1}^S\mu_i\hat{e}_i$ can be computed as $\tau_{es}^2 = \text{Var}\left[ \hat{e}_s - \frac{\sum_{i=1}^S\mu_i\hat{e}_i}{\mu_s S} \right] = \left(1 - \frac{2}{S}\right)\sigma_s^2 + \frac{\sum_{i=1}^S\mu^2_i\sigma^2_i}{\mu_s^2 S^2}$. Note that we need $\sum_{s=1}^S \mu_s^2 \tau_{gs}^2 = \tau_\mrm{pool}^2$. One solution is to set $\tau_{gs}^2 = \frac{1}{\mu_s^2 S} \tau_\mrm{pool}^2$. Using the constraint $\tau_{es}^2 + \tau_{gs}^2 = \tau_s^2$ and the expressions for $\tau_{es}^2$ and $\tau_{gs}^2$, we have $\left(1 - \frac{1}{S}\right)^2\sigma_s^2 + \frac{1}{\mu_s^2 S^2}\sum_{i \neq s}\mu^2_i\sigma^2_i = \tau_s^2 - \frac{1}{\mu_s^2 S} \tau_\mrm{pool}^2$. We can write this expression for all $s \in [S]$ in matrix form and solve for $\left[\sigma_1^2\ \sigma_2^2\ \ldots \ \sigma_S^2\right]^\top$ as
\[
\begin{bmatrix}
    \left(1 - \frac{1}{S}\right)^2 & \frac{\mu_2^2}{\mu_1^2 S^2} & \cdots & \frac{\mu_S^2}{\mu_1^2 S^2} \\
    \frac{\mu_1^2}{\mu_2^2 S^2} & \left(1 - \frac{1}{S}\right)^2 & \cdots & \frac{\mu_S^2}{\mu_2^2 S^2} \\
    \vdots                      & \vdots                         & \ddots & \vdots \\
    \frac{\mu_1^2}{\mu_S^2 S^2} & \frac{\mu_2^2}{\mu_S^2 S^2}    & \cdots & \left(1 - \frac{1}{S}\right)^2
\end{bmatrix}^{-1}
\begin{bmatrix}
    \tau_1^2 - \frac{\tau_\mrm{pool}^2}{\mu_1^2 S}\\
    \tau_2^2 - \frac{\tau_\mrm{pool}^2}{\mu_2^2 S}\\
    \vdots                      \\
    \tau_S^2 - \frac{\tau_\mrm{pool}^2}{\mu_S^2 S} 
    \end{bmatrix} 
\]

\noindent\textbf{Privacy Analysis in Asymmetric Setting. }We present an analysis of privacy for the aforementioned scheme in asymmetric setting. Recall that the adversary can observe $\hat{\vect{a}} = \left[\hat{a}_1, \ldots, \hat{a}_{S_H}\right]^\top \in \mathbb{R}^{S_H}$ and $\hat{e} = \sum_{s \in \mathbb{S}_H} \hat{e}_s$. In other words, the adversary observes the vector $\vect{y} = \left[\hat{\vect{a}}^\top, \hat{e}\right]^\top \in \mathbb{R}^{S_H + 1}$ to make inference about the non-colluding sites. As before, we must show that $\left|\log\frac{g(\vect{y} | \vect{a})}{g(\vect{y} | \vect{a}')}\right| \leq \epsilon$ holds with probability (over the randomness of the mechanism) at least $1-\delta$ for guaranteeing differential privacy. 
Recall that we release $\hat{a}_s = f(\vect{x}_s) + e_s + g_s$ from each site. We observe $\mathbb{E}(\hat{a}_s) = f(\vect{x}_s),\ \mrm{Var}(\hat{a}_s) = \tau_s^2,\ \forall s \in [S]$ and $\mathbb{E}(\hat{a}_{s_1} \hat{a}_{s_2}) = f(\vect{x}_{s_1}) f(\vect{x}_{s_2}) - \frac{\mu_{s_1}\sigma_{s_1}^2}{\mu_{s_2}S} - \frac{\mu_{s_2}\sigma_{s_2}^2}{\mu_{s_1}S} + \frac{1}{\mu_{s_1}\mu_{s_2}S^2} \sum_{i = 1}^S \mu_i^2\sigma_i^2,\ \forall s_1 \neq s_2 \in [S]$. Without loss of generality, we can assume~\cite{dwork2013algorithmic} that $\vect{a} = \vect{0}$ and $\vect{a}' = \vect{a} - \vect{v}$, where $\vect{v} = \left[\frac{1}{N_s},0, \ldots, 0\right]^\top$. That is, the random variable $\hat{\vect{a}}$ is $\mathcal{N}(\vect{0}, \Sigma_{\hat{\vect{a}}})$, where 
\[
\Sigma_{\hat{\vect{a}}} = 
\begin{bmatrix}
    \tau_1^2    & \Psi(1,2) & \cdots & \Psi(1,S) \\
    \Psi(2,1)   & \tau_2^2  & \cdots & \Psi(2,S) \\
    \vdots      & \vdots    & \ddots & \vdots \\
    \Psi(S,1)   & \Psi(S,2) & \cdots & \tau_{S_H}^2
\end{bmatrix},
\]
and $\Psi(i,j) = -\frac{1}{S}\left(\frac{\mu_i\sigma_i^2}{\mu_j} + \frac{\mu_j\sigma_j^2}{\mu_i}\right) + \frac{\sum_{s = 1}^S \mu_s^2\sigma_s^2}{\mu_i\mu_jS^2}$. Additionally, the random variable $\hat{e}$ is $\mathcal{N}(0, \tau^2_{\hat{e}})$, where $\tau^2_{\hat{e}} = \sum_{s = 1}^{S_H} \sigma_s^2$. Therefore, $g(\vect{y} | \vect{a})$ is the density of $\mathcal{N}(\vect{0}, \Sigma)$, where $\Sigma =
\begin{bmatrix}
    \Sigma_{\hat{\vect{a}}} 										& \Sigma_{\hat{\vect{a}}\hat{e}} \\
    \Sigma_{\hat{\vect{a}}\hat{e}}^\top			& \tau^2_{\hat{e}}
\end{bmatrix} \in \mathbb{R}^{(S_H + 1) \times (S_H + 1)}$. With some simple algebra, we can find the expression for each entry of $\Sigma_{\hat{\vect{a}}\hat{e}} \in \mathbb{R}^{S_H}$: $\left[\Sigma_{\hat{\vect{a}}\hat{e}}\right]_s = \sigma_s^2 - \frac{1}{\mu_s S}\sum_{i = 1}^{S_H}\mu_i^2 \sigma_i^2$.
The rest of the proof proceeds as the proof of Theorem~\ref{thm:cape}. Note that, due to the complex nature of the expression of $\Sigma$, we do not have a closed form solution for $\mu_z$ and $\sigma_z$ (but we can numerically compute the values and thus, the resulting $\delta$).

\section{Improved Differentially Private djICA}\label{sec:dp_djica}
In this section, we propose an algorithm that improves upon our previous decentralized DP \djICA\ algorithm~\cite{imtiaz2016} and achieves the same noise variance as the DP pooled-data scenario in certain regimes. Recall that we are considering the joint ICA (jICA)~\cite{calhoun2006} of decentralized fMRI data, which assumes a global mixing process (common spatial maps). More specifically, the global mixing matrix $\matr{A} \in \mathbb{R}^{D\times R}$ is assumed to generate the time courses in $\matr{X}_s$ from underlying sources $\matr{S}_s \in \mathbb{R}^{R\times N_s}$ at each site $s\in[S]$. Each site has data from $M_s$ individuals, which are concatenated temporally to form the local data matrix $\matr{X}_s \in \mathbb{R}^{D\times N_s}$. That is: $\matr{X} = \left[ \matr{A} \matr{S}_1 \ldots \matr{A} \matr{S}_S \right] \in \mathbb{R}^{D\times N}$. We estimate the DP global unmixing matrix $\matr{W} \in \mathbb{R}^{R\times D} \approx \matr{A}^+$ by solving the Infomax ICA problem (see Section \ref{sec:problem_formulation}) in the decentralized setting with a multi-round gradient descent that employing $\cape$.

Neuroimaging data is generally very high dimensional. We therefore use the recently proposed~\cite{imtiaz2018} DP decentralized PCA algorithm ($\capepca$) as an efficient and privacy-preserving dimension-reduction step of our proposed $\capedjica$ algorithm. 
For simplicity, we assume that the observed samples are mean-centered. 
We present a slightly modified version of the original $\capepca$ algorithm in Algorithm \ref{alg:dist_dpca} (Appendix~\ref{appendix:capepca}) to match the robust $\cape$ scheme from Section \ref{sec:cape}. Note that the scheme proposed in~\cite{imtiazDPCA2018} was limited by the larger variance of the additive noise at the local sites due to the smaller sample size. The $\capepca$ alleviates this problem using the $\cape$ scheme and achieves the same noise variance as the pooled-data scenario in the symmetric setting. 

Let the output of $\capepca$ to be $\matr{V}_R \in \mathbb{R}^{D\times R}$, which is sent to the sites from the aggregator. Then the reduced dimensional ($R\times N_s$) data matrix at site $s$ is denoted by: $\matr{X}_s^r = \matr{V}_R^\top \matr{X}_s$. These projected samples are the inputs to the proposed $\capedjica$ algorithm that estimates the unmixing matrix $\matr{W}$ through a gradient descent~\cite{amari1996}. Our proposed $\capedjica$ algorithm employs the $\cape$ protocol to perform the privacy-preserving iterative message-passing between sites and the aggregator to solve for $\matr{W}$. We start the algorithm by initializing $\matr{W}$. At each iteration $j$, the sites adjust the local source estimates $\matr{Z}_s(j) = \matr{W}(j-1) \matr{X}_s$ by their bias estimate $\vect{b}(j-1) \vect{1}^\top$. Local gradients of the \emph{empirical average} loss function are computed with respect to $\vect{W}$ and $\vect{b}$~\cite{baker2018}. More specifically, the gradient with respect to $\matr{W}$ at site $s$ is given~\cite{baker2018} by $\matr{G}_s = \frac{1}{N_s} \left(N_s\matr{I} + \left(\vect{1} - 2\matr{Y}_s\right)\matr{Z}_s^\top\right)\matr{W}$, where $\matr{Z}_s = \matr{W}\matr{X}_s^r + \vect{b}\vect{1}^\top$, $\matr{Y}_s = g\left(\matr{Z}_s\right)$; $\vect{b} \in \mathbb{R}^R$ is the bias and $\vect{1}$ is a vector of ones. If we denote $\vect{1} - 2\matr{Y}_s$ with $\hat{\matr{Y}}_s$ then we have $\matr{G}_s  = \frac{1}{N_s}\left(N_s\matr{I} + \hat{\matr{Y}}_s\matr{Z}_s^\top\right)\matr{W} = \frac{1}{N_s} \sum_{n=1}^{N_s} \left(\matr{I} + \hat{\vect{y}}_{s,n}\vect{z}_{s,n}^\top\right)\matr{W}$, where $\left(\matr{I} + \hat{\vect{y}}_{s,n}\vect{z}_{s,n}^\top\right)\matr{W}$ is the gradient contribution of one time point of a subject's data matrix. Note that this gradient estimate is needed to be sent to the aggregator from the site. Therefore, we need to approximate this gradient satisfying differential privacy. To that end, let us consider that the norm of the gradient due to each subject is bounded by $B_G N_m$, where $N_m$ is the time course length for each subject's fMRI scan and $B_G$ is some constant. This implies that $\left\|\left(\matr{I} + \hat{\vect{y}}_{s,n}\vect{z}_{s,n}^\top\right)\matr{W}\right\|_F \leq B_G$. It is easy to see that by changing one subject (i.e., for a neighboring dataset), the gradient at site $s$ can change by at most $\frac{2B_G N_m}{N_s} = \frac{2B_G}{M_s}$. Therefore, the $\mathcal{L}_2$ sensitivity of the function $f(\matr{X}_s) = \matr{G}_s$ is $\Delta_G^s = \frac{2B_G}{M_s}$. In addition to the unmixing matrix $\matr{W}$, we update a bias term $\vect{b}$ using a gradient descent~\cite{baker2018}. The gradient of the empirical average loss function with respect to the bias at site $s$ is given~\cite{baker2018} by $\vect{h}_s = \frac{1}{N_s} \sum_{n=1}^{N_s} \hat{\vect{y}}_{s,n}$. Similar to the case of $\matr{G}_s$, we can find the $\mathcal{L}_2$ sensitivity of the function $f(\matr{X}_s) = \vect{h}_s$ as $\Delta_h^s = \frac{2B_h}{M_s}$, where $\|\hat{\vect{y}}_{s,n}\|_2 \leq B_h$. Note that for other neighborhood definitions (time point level instead of subject level), one should consider the temporal correlation in the data~\cite{cao2017}. According to the Gaussian mechanism~\cite{dwork2006}, computing $(\epsilon, \delta)$ DP approximates of $\matr{G}_s$ and $\vect{h}_s$ requires noise standard deviations $\tau_G^s$ and $\tau_h^s$ satisfy
\begin{align}\label{eqn:grad_noise_variance}
\tau_G^s &= \frac{\Delta_G^s}{\epsilon}\sqrt{2\log\frac{1.25}{\delta}},\ \tau_h^s = \frac{\Delta_h^s}{\epsilon}\sqrt{2\log\frac{1.25}{\delta}}.
\end{align}
As mentioned before, we employ the $\cape$ protocol to combine the gradients from the sites at the aggregator to achieve the same utility level as that of the pooled data scenario. More specifically, each site generates two noise terms: $\matr{E}_s^G \in \mathbb{R}^{R\times R}$ and $\vect{e}_s^h \in \mathbb{R}^R$, collectively among all sites (element-wise, according to Algorithm \ref{alg:zero-sum-noise-generation}) \hypertarget{generate_E}{at each iteration round}. Additionally, each site $s$ generates the following two noise terms locally at each \hypertarget{generate_K}{iteration}:
\begin{itemize}
\item $\matr{K}_s^G \in \mathbb{R}^{R\times R}$; $[\matr{K}_s^G]_{ij}$ i.i.d. $\sim \mathcal{N}(0, \tau_{Gk}^2)$; $\tau_{Gk}^2 = \frac{1}{S}{\tau_G^s}^2$
\item $\vect{k}_s^h \in \mathbb{R}^R$; $[\vect{k}_s^h]_i$ i.i.d. $\sim \mathcal{N}(0, \tau_{hk}^2)$; $\tau_{hk}^2 = \frac{1}{S}{\tau_h^s}^2$.
\end{itemize}
At each iteration round, the sites compute the noisy estimates of the gradients of $\matr{W}$ and $\vect{b}$: $hat{\matr{G}}_s = \matr{G}_s + \matr{E}_s^G + \matr{K}_s^G$, $\hat{\vect{h}}_s = \vect{h}_s + \vect{e}_s^h + \vect{k}_s^h$. These two terms are then sent to the aggregator and the aggregator computes: $\Delta_\matr{W} = \rho \frac{1}{S} \sum^{S}_{s=1} \hat{\matr{G}}_s$ and $\Delta_\vect{b} = \rho \frac{1}{S} \sum^{S}_{s=1} \hat{\vect{h}}_s$, where $\rho$ is the learning rate. These gradient estimates are then used to update the variables $\matr{W}$ and $\vect{b}$. By Lemma~\ref{lemma:cape}, the variances of the noise of the two estimates: $\Delta_\matr{W}$ and $\Delta_\vect{b}$, are exactly the same as the pooled-data scenario in the symmetric setting. The complete algorithm is shown in Algorithm \ref{alg:capedjica} in Appendix~\ref{appendix:capedjica}.

Note that, one does not need to explicitly find the bounds $B_G$ and $B_h$. Instead, the gradients due to each subject can be clipped to some pre-determined $B_G N_m$ or $B_h N_m$ in $\mathcal{L}_2$ norm sense (where $N_m$ is known apriori from the data collection stage). That is, we can replace $\matr{G}_{s,n} = \left(\matr{I} + \hat{\vect{y}}_{s,n}\vect{z}_{s,n}^\top\right)\matr{W}$ with $\matr{G}_{s,n} = \frac{\matr{G}_{s,n}}{\max\left(1, \frac{\|\matr{G}_{s,n}\|_F}{B_G}\right)}$. Similarly, we can replace $\vect{h}_{s,n} = \hat{\vect{y}}_{s,n}$ with $\vect{h}_{s,n} = \frac{\vect{h}_{s,n}}{\max\left(1, \frac{\|\vect{h}_{s,n}\|_2}{B_h}\right)}$.

\begin{Rem}[Consequences of Norm Clipping]\label{rem:norm-clipping}
The norm clipping destroys the unbiasedness of the gradient estimate~\cite{abadi2016}. If we choose $B_G$ and $B_h$ to be too small, the average clipped gradient may be a poor estimate of the true gradient. Moreover, $B_G$ and $B_h$ dictates the additive noise level. In general, clipping prescribes taking a smaller step ``downhill'' towards the optimal point~\cite{bassily2014} and may slow down the convergence.
\end{Rem}

\subsection{Privacy Analysis using R\'enyi Differential Privacy}\label{sec:djica-renyi-dp}

We now analyze the $\capedjica$ algorithm with R\'enyi Differential Privacy~\cite{mironov2017}. Analyzing the total privacy loss of a multi-shot algorithm, each stage of which is DP, is a challenging task. It has been shown~\cite{abadi2016, mironov2017} that the advanced composition theorem~\cite{dwork2013algorithmic} for $(\epsilon, \delta)$-differential privacy can be loose. The main reason is that one can formulate infinitely many $(\epsilon, \delta)$-DP algorithms for a given noise variance $\tau^2$. RDP offers a much simpler composition rule that is shown to be tight~\cite{mironov2017}. We review some necessary properties of RDP in Appendix~\ref{appendix:background}. Recall that at each iteration $j$ of $\capedjica$, we compute the noisy estimates of the gradients: $\Delta_\matr{W}(j)$ and $\Delta_\vect{b}(j)$. As we employed the $\cape$ scheme in the symmetric setting, the variances of noise at the aggregator for $\Delta_\matr{W}(j)$ and $\Delta_\vect{b}(j)$ are: $\sigma^2_\matr{W} = \frac{\rho^2 {\tau_G^\mrm{pool}}^2}{\Delta_G}$ and $\sigma^2_\vect{b} = \frac{\rho^2 {\tau_h^\mrm{pool}}^2}{\Delta_h}$, respectively, where $\Delta_G = \frac{\Delta_G^s}{S}$ and $\Delta_h = \frac{\Delta_h^s}{S}$.
From Proposition \ref{prop:rdp_gauss_mech}, we have that the computation of $\Delta_\matr{W}(j)$ is $\left(\alpha, \alpha/\left(2\sigma^2_\matr{W}\right)\right)$-RDP. Similarly, the computation of $\Delta_\vect{b}(j)$ is $\left(\alpha, \alpha/\left(2\sigma^2_\vect{b}\right)\right)$-RDP. By Proposition \ref{prop:composition_rdp}, we have that each iteration step of $\capedjica$ is $\left(\alpha, \frac{\alpha}{2}\left(\frac{1}{\sigma^2_\matr{W}} + \frac{1}{\sigma^2_\vect{b}}\right)\right)$-RDP. Denoting the number of required iterations for convergence by $J^*$ then, under $J^*$-fold composition of RDP, the overall $\capedjica$ algorithm is $(\alpha, \frac{\alpha J^*}{2\sigma^2_\mrm{RDP}})$-RDP, where $\frac{1}{\sigma^2_\mrm{RDP}} = \left(\frac{1}{\sigma^2_\matr{W}} + \frac{1}{\sigma^2_\vect{b}}\right)$.
From Proposition \ref{prop:rdp_to_dp}, we can conclude that the $\capedjica$ algorithm satisfies $\left(\frac{\alpha J^*}{2\sigma^2_\mrm{RDP}} + \frac{\log\frac{1}{\delta_r}}{\alpha-1}, \delta_r\right)$-differential privacy for any $0 < \delta_r < 1$. For a given $\delta_r$, we find the optimal $\alpha_\mrm{opt}$ as: $\alpha_\mrm{opt} = 1 + \sqrt{\frac{2}{J^*}\sigma^2_\mrm{RDP}\log\frac{1}{\delta_r}}$.
Therefore, $\capedjica$ algorithm is $\left(\frac{\alpha_\mrm{opt} J^*}{2\sigma^2_\mrm{RDP}} + \frac{\log\frac{1}{\delta_r}}{\alpha_\mrm{opt} - 1}, \delta_r\right)$-DP for any $0 < \delta_r < 1$.

\subsection{Privacy Accounting using Moments Accountant}\label{sec:djica-moments-accountant}
In this section, we use the moments accountant~\cite{abadi2016} framework to compute the overall privacy loss of our $\capedjica$ algorithm. Moments accountant can be used to achieve a much smaller overall $\epsilon$ than the strong composition theorem~\cite{dwork2013algorithmic}. As mentioned before, na\"ively employing the additive nature of the privacy loss results in the worst case analysis, i.e., assumes that each iteration step exposes the worst privacy risk and this exaggerates the total privacy loss. However, in practice, the privacy loss is a random variable that depends on the dataset and is typically well-behaved (concentrated around its expected value). Due to space constraints, we presented the detailed analysis of $\capedjica$ in Appendix~\ref{appendix:moments-accountant}. Briefly, we can formulate a quadratic equation in terms of $\epsilon$ and then find the best $\epsilon$ for a given $\delta_\mrm{target}$: $\frac{\sigma^2}{2J^*\Delta^2}\epsilon^2 - \epsilon + \frac{J^*\Delta^2}{8\sigma^2} + \log\delta_\mrm{target} = 0$. Here, the noise variance $\sigma^2$ consists of two parts: $\sigma^2_\matr{W}$ and $\sigma^2_\vect{b}$ for the $\capedjica$ algorithm.

\subsection{Performance Improvement with Correlated Noise}\label{sec:dist_djica_perf_gain}
The existing DP \djICA\ algorithm~\cite{imtiaz2016} achieved $J^*\epsilon$-differential privacy (where $J^*$ is the total number of iterations required for convergence) by adding a noise term to the local estimate of the source (i.e., $\matr{Z}_s(j)$). Although the algorithm offered a ``pure'' DP \djICA\ procedure, there are a few shortcomings. The cost of achieving pure differential-privacy (i.e., employing the Laplace mechanism~\cite{dwork2006}) was that the neighboring dataset condition was met by restricting the $\mathcal{L}_2$-norm of the samples to satisfy $\|\vect{x}_n\|_2 \leq \frac{1}{2\sqrt{D}}$, which can be too limiting for datasets with large ambient dimensions. The effect of this is apparent from the experiments. Last but not the least, the DP PCA preprocessing step was less fault tolerant because of the \emph{pass the parcel} or \emph{cyclic} style message passing among the sites, where site dropouts are more drastic than the one employed in this paper (a certain number of site dropouts is permitted~\cite{Bonawitz17}). By employing the $\cape$ protocol in the preprocessing stage and also in the optimization process, we expect to gain a significant performance boost. We validate the performance gain in the Experimental Results (Section \ref{sec:djica-experimental-results}).

\noindent\textbf{Convergence of $\capedjica$ Algorithm.} We note that the gradient estimate at the aggregator (Step \ref{alg:dp_djica:gradW} in Algorithm \ref{alg:capedjica}) essentially contains the noise $\frac{\rho}{S}\sum_{s=1}^S \matr{K}_s^G$, which is zero mean. Therefore, in expectation, the estimate of the gradient converges to the true gradient~\cite{bottou1999}. However, if the batch size is too small, the noise can be too high for the algorithm to converge~\cite{AnandSGD}. Since the total additive noise variance is smaller for $\capedjica$ than the conventional case by a factor of $S$, the convergence rate is faster.

\noindent\textbf{Communication Cost of $\capedjica$.} We analyze the total communication cost associated with the proposed $\capedjica$ algorithm. At each iteration round, we need to generate two zero-sum noise terms, which entails $O(S + R^2)$ communication complexity of the sites and $O(S^2 + SR^2)$ communication complexity of the aggregator~\cite{Bonawitz17}. Each site computes the noisy gradient and sends one $R\times R$ matrix and one $R$ dimensional vector to the aggregator. And finally, the aggregator sends the $R\times R$ updated weight matrix and $R$ dimensional bias estimate to the sites. The total communication cost is $O(S + R^2)$ for the sites and $O(S^2 + SR^2)$ for the central node. This is expected as we are estimating an $R\times R$ matrix in a decentralized setting.

\section{Experimental Results}\label{sec:djica-experimental-results}
\begin{figure*}[t]
  \centering
  \includegraphics[width=0.95\textwidth]{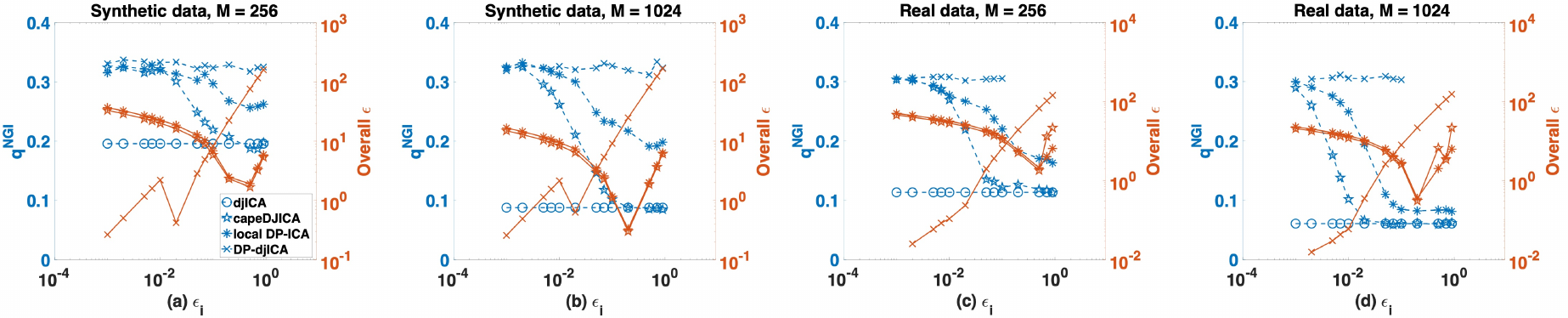}\\
  \vspace{-0.1in}
  \caption{Variation of $q^\mrm{NGI}$ and overall $\epsilon$ with privacy parameter $\epsilon_i$ for: (a)-(b) synthetic fMRI data, (c)-(d) real fMRI data. For $\capedjica$, higher $\epsilon_i$ results a smaller $q^\mrm{NGI}$, but not necessarily a small overall $\epsilon$, i.e., an optimal $\epsilon_i$ can be chosen based on $q^\mrm{NGI}$ or overall $\epsilon$ requirement.}
  \vspace{-0.1in}
  \label{fig:real_synth_eps}
\end{figure*}

\begin{figure*}[t]
  \centering
  \includegraphics[width=0.95\textwidth]{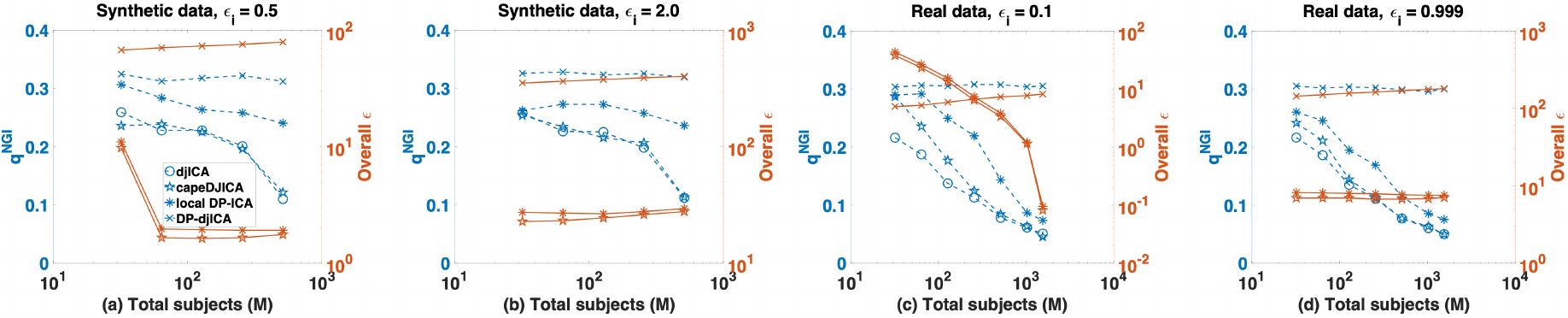}\\
  \vspace{-0.1in}
  \caption{Variation of $q^\mrm{NGI}$ and overall $\epsilon$ with total number of subjects $M$ for: (a)-(b) synthetic fMRI data, (c)-(d) real fMRI data. For $\capedjica$, higher $M$ results a smaller $q^\mrm{NGI}$ and a smaller overall $\epsilon$.}
  \vspace{-0.2in}
  \label{fig:real_synth_M}
\end{figure*}

In this section, we empirically show the effectiveness of the proposed $\capedjica$ algorithm. We note the intricate relationship between $\epsilon$ and $\delta$ (see Theorem \ref{thm:capedjica}) due to the correlated noise scheme and the challenge of characterizing the overall privacy loss in our multi-round $\capedjica$ algorithm. We designed the experiments to better demonstrate the trade-off between performance and several parameters: $\epsilon$, $\delta$ and $M$. We show the simulation results to compare the performance of our $\capedjica$ algorithm with the existing DP \djICA\ algorithm~\cite{imtiaz2016} ($\dpdjica$), the non-private \djICA\ algorithm~\cite{baker2018} and a DP ICA algorithm operating on only local data ($\local$). We modified the base non-private \djICA\ algorithm to incorporate the gradient bounds $B_G$ and $B_h$. Although we are proposing an algorithm for decentralized setting, we included the performance indices for the local setting to show the effect of smaller sample sizes on the performance. We note that the $\dpdjica$ algorithm~\cite{imtiaz2016} offers $\epsilon$-differential privacy as opposed to $(\epsilon, \delta)$-differential privacy offered by $\capedjica$. For both synthetic and real datasets, we consider the symmetric setting (i.e., $N_s = \frac{N}{S}$, $\tau_G^s = \tau_G$ and $\tau_h^s = \tau_h$). We limited the maximum number of iterations $J$ to be 1000 (however, the number of iterations varies with the algorithm and amount of noise). We chose the norm bounds $B_G = 30$, $B_h = \sqrt{B_G}$, number of sites $S=4$, $S_C = \ceil*{\frac{S}{3}} - 1$, the target $\delta = 10^{-5}$ and the learning rate $\rho = \frac{0.015}{\log (R)}$. We show the average performance over 10 independent runs. \textcolor{blue}{Note that, the choice of hyper-parameters is non-trivial~\cite{Chaudhuri2013} and corresponding end-to-end privacy analysis is still an open problem.}

\noindent\textbf{Synthetic Data. }We generated the synthetic data from the same model as~\cite{baker2018}. The source signals $\matr{S}$ were simulated using the generalized autoregressive (AR) conditional heteroscedastic (GARCH) model \cite{engle1982, bollerslev1986}. We used $M = 1024$ simulated subjects in our experiments. For each subject, we generated $R=20$ time courses with 250 time points. The data samples are equally divided into $S = 4$ sites. For each subject, the fMRI images are $30\times 30$ dimensional. We employ the $\capepca$ algorithm~\cite{imtiaz2018} (Appendix~\ref{appendix:capepca}) as a preprocessing stage to reduce the sample dimension from $D = 900$ to $R = 20$. The $\capedjica$ is carried out upon the $R$-dimensional samples. 

\noindent\textbf{Real Data. } We use the same data and preprocessing as Baker et al.~\cite{baker2018}: the data were collected using a 3-T Siemens Trio scanner with a 12-channel radio frequency coil, according to the protocol in Allen et al.~\cite{allen2011}. In the dataset, the resting-state scan durations range from 2 min 8 sec to 10 min 2 sec, with an average of 5 min 16 sec~\cite{baker2018}. We used a total of $M = 1548$ subjects from the dataset and estimated $R=50$ independent components using the algorithms under consideration. For details on the preprocessing, please see~\cite{baker2018}. We also projected the data onto a 50-dimensional PCA subspace estimated using pooled non-private PCA. As we do not have the ground truth for the real data, we computed a \emph{pseudo} ground truth~\cite{baker2018} by performing a pooled non-private  analysis on the data and estimating the unmixing matrix. The performance of $\capedjica$, \djICA, $\dpdjica$ and $\local$ algorithms are evaluated against this pseudo ground truth. 

\noindent\textbf{Performance Index. }We set $\tau_G^s = \frac{\Delta_G^s}{\epsilon_i}\sqrt{2\log \frac{1.25}{10^{-2}}}$ and $\tau_h^s = \frac{\Delta_h^s}{\epsilon_i}\sqrt{2\log \frac{1.25}{10^{-2}}}$ for our experiments, where $\epsilon_i$ is the privacy parameter per iteration, $\Delta_G^s$ and $\Delta_h^s$ are the $\mathcal{L}_2$ sensitivities of $\matr{G}_s$ and $\vect{h}_s$, respectively. To evaluate the performance of the algorithms, we consider the quality of the estimated unmixing matrix $\matr{W}$. More specifically, we utilize the normalized gain index $q^\mrm{NGI}$~\cite{baker2018, oja2011} that quantizes the quality of $\matr{W}$. The normalized gain index $q^\mrm{NGI}$ varies from 0 to 1, with lower values indicating a better estimation of a set of ground-truth components (i.e. the unmixing matrix times the mixing matrix is closer to an identity matrix~\cite{oja2011}). For practical usability of the recovered $\matr{A}$, we need to achieve $q^\mrm{NGI} \leq 0.1$~\cite{baker2018}. We consider the overall $\epsilon$ as a performance index. We plotted the overall $\epsilon$ (with solid lines on the right $y$-axis) along with $q^\mrm{NGI}$ (with dashed lines on the left $y$-axis) as a means for visualizing how the privacy-utility trade-off varies with different parameters. For a given privacy budget (performance requirement), the user can use the overall $\epsilon$ plot on the right $y$-axis, shown with solid lines, ($q^\mrm{NGI}$ plot on the left $y$-axis, shown with dashed lines) to find the required $\epsilon_i$ or $M$ on the $x$-axis and thereby, find the corresponding performance (overall $\epsilon$). We computed the overall $\epsilon$ for the $\capedjica$ and the $\local$ algorithms using the RDP technique (Section~\ref{sec:djica-renyi-dp}) and for the $\dpdjica$ algorithm using the composition theorem~\cite{dwork2013algorithmic}. Note that, we are reporting the privacy spent during the course of the gradient descent. The total privacy spent including the PCA would be slightly higher.

\begin{figure*}[t]
  \centering
  \includegraphics[width=0.75\textwidth]{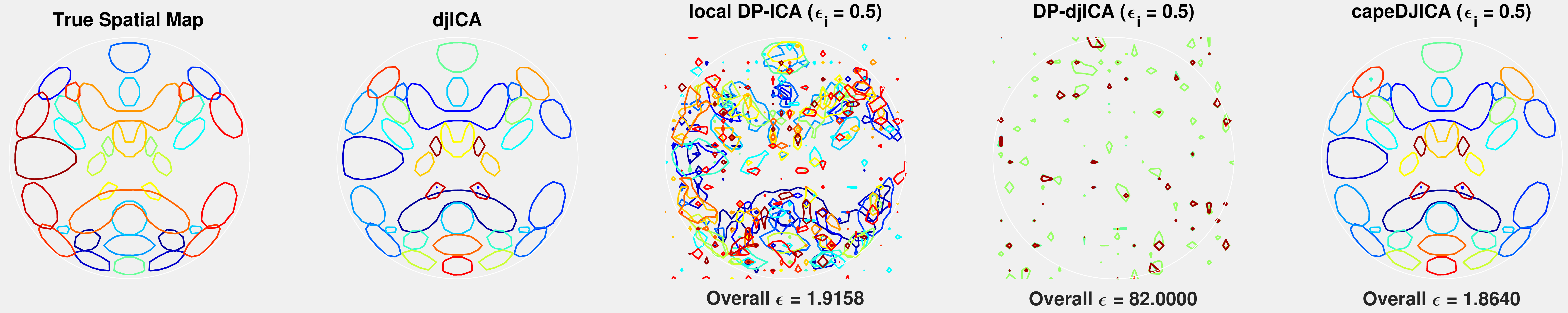}\\
  \vspace{-0.1in}
  \caption{Recovered spatial maps from synthetic data: the ground truth and the ones resulting from \djICA, $\local$, $\dpdjica$ and $\capedjica$.}
  \vspace{-0.1in}
  \label{fig:spatial_maps_orig_and_djica}
\end{figure*}

\begin{figure*}[t]
  \centering
  \includegraphics[width=0.95\textwidth]{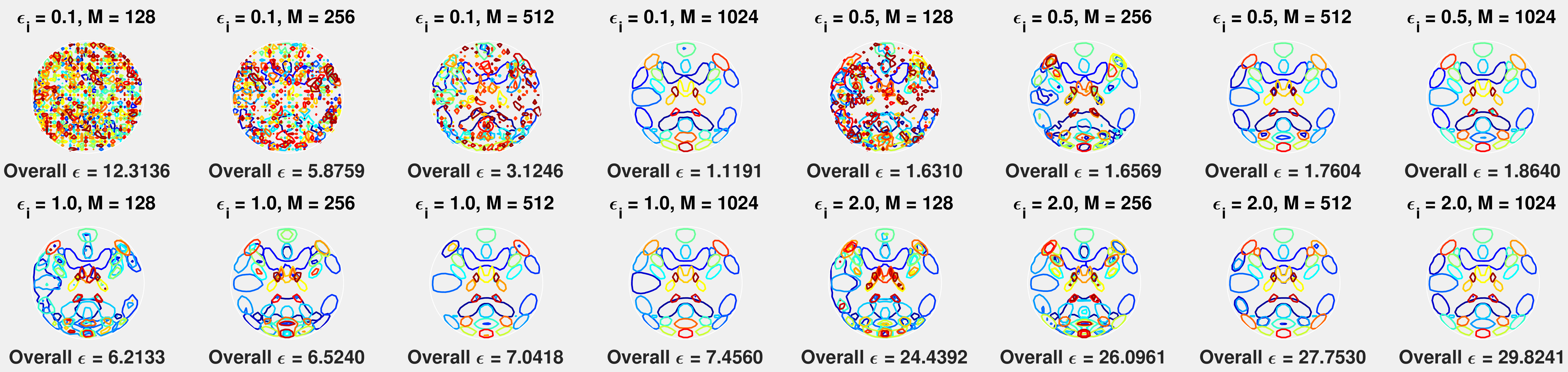}\\
  \vspace{-0.1in}
  \caption{Spatial maps (synthetic data) resulting from $\capedjica$ for different parameters. $\capedjica$ estimates spatial maps that closely resemble the true ones, even for strict privacy guarantee (small overall $\epsilon$).}
  \vspace{-0.2in}
  \label{fig:spatial_maps_eps_M_table}
\end{figure*}

\noindent\textbf{Performance Variation with $\epsilon$.} First, we explore how the privacy-utility tradeoff between $q^\mrm{NGI}$ and the overall ``privacy risk'' $\epsilon$ varies with $\epsilon_i$. As mentioned before, we compare the performance of $\capedjica$ with those of the \djICA, the $\dpdjica$ and $\local$.

In Figs. \ref{fig:real_synth_eps}(a) - (d), we show the variation of $q^\mrm{NGI}$ and overall $\epsilon$ for different algorithms with $\epsilon_i$ on synthetic and real data. For both datasets, we show the performance indices for two different $M$ values, namely $M = 256$ and $M = 1024$. We observe from the figures that the proposed $\capedjica$ outperforms the existing $\dpdjica$ by a large margin for the range of $\epsilon_i$ values that results in $q^\mrm{NGI} \leq 0.1$. This is expected as $\dpdjica$ suffers from too much noise (see Section \ref{sec:dist_djica_perf_gain} for the explanation). $\capedjica$ also guarantees the smallest overall $\epsilon$ among the privacy-preserving methods. $\capedjica$ can reach the utility level of the non-private \djICA\ for some parameter choices and naturally outperforms $\local$ as estimation of the sources is much accurate when more samples are available. For the same privacy loss (i.e., for a fixed $\epsilon$), one can achieve better performance by increasing the number of subjects. For both synthetic and real data, we note that assigning a higher $\epsilon_i$ may provide a good $q^\mrm{NGI}$ but does not guarantee a small overall $\epsilon$. The user needs to choose the $\epsilon_i$ based on the ``privacy budget'' and the required performance.

\noindent\textbf{Performance Variation with $M$.} Next, in Figs. \ref{fig:real_synth_M}(a) - (d), we show the variation of $q^\mrm{NGI}$ and the overall $\epsilon$ with the total number of subjects $M$ for two different $\epsilon_i$ values on synthetic and real data. We observe similar trends in performance as in the case of varying $\epsilon_i$. The $\capedjica$ algorithm outperforms the $\dpdjica$ and the $\local$: with respect to both $q^\mrm{NGI}$ and the overall $\epsilon$. For the $q^\mrm{NGI}$, the $\capedjica$ performs very closely to the non-private \djICA, even for moderate $M$ values, while guaranteeing the smallest overall $\epsilon$. The performance gain over $\dpdjica$ is particularly noteworthy. For a fixed number of subjects, increasing $\epsilon$ results in a slightly better utility, albeit at the cost of greater privacy loss. We show the performance variation with $\delta$ in Appendix \ref{appendix:perf-delta}.

\noindent\textbf{Reconstructed Spatial Maps.} Finally, we intend to demonstrate how the estimated spatial maps (the estimated global mixing matrix $\matr{A}$, see Section~\ref{sec:problem_formulation}) actually look like, as interpretability is one of the most important concerns for fMRI applications. In Figure \ref{fig:spatial_maps_orig_and_djica}, we show the true spatial map, the ones estimated from the non-private \djICA~\cite{baker2018}, $\local$, $\dpdjica$ and $\capedjica$ algorithms. It is evident from the figure that the spatial map recovered by the proposed $\capedjica$ is very close to that of the ground truth. The overall $\epsilon$ is also very small. The $\local$, although can achieve a small $\epsilon$, cannot recover the spatial maps well enough for practical purposes. However, increasing the $\epsilon_i$ and/or increasing the number of subjects would certainly improve the quality of the spatial maps. Finally, for the $\dpdjica$, the algorithm fails to converge to anything meaningful due to excessive amount of noise.

In Figure \ref{fig:spatial_maps_eps_M_table}, we show the estimated spatial maps resulting from the proposed $\capedjica$ algorithm along with the overall $\epsilon$ for a variety of combinations of $\epsilon_i$ and $M$. We observe that when sufficiently large number of subjects are available, the estimated spatial maps closely resemble the true one, even for strict privacy guarantee (small overall $\epsilon$). For smaller number of samples, we may need to compensate by allowing larger $\epsilon$ values to achieve good utility. In general, we observe that $\capedjica$ can achieve very good approximate to the true spatial map, almost indistinguishable from the non-private spatial map. This emphasizes the effectiveness of the proposed $\capedjica$ in the sense that very meaningful utility can be achieved even with strict privacy guarantee.

\section{Conclusion}\label{sec:conclusion}

We proposed a novel decentralized DP computation scheme, $\cape$, which is best suited for ML applications involving locally held private data. Example scenarios include health care research with legal and ethical limitations on the degree of sharing the ``raw'' data. $\cape$ can greatly improve the privacy-utility tradeoff when (a) all parties follow the protocol and (b) the number of colluding sites is not more than $\ceil*{S/3} - 1$. Our proposed $\cape$ protocol is based on an estimation-theoretic analysis of the noise addition process for differential privacy and therefore, provides different guarantees than cryptographic approaches such as SMC. $\cape$ can achieve the same level of additive noise variance as the pooled data scenario in certain regimes (assuming the availability of some reasonable resources), and can be extended to asymmetric network/privacy settings. We analytically show that $\cape$ can be applied to numerous decentralized ML problems of interest. 
Leveraging the effectiveness of the $\cape$ scheme, we proposed a new and improved algorithm for DP decentralized joint independent component analysis ($\capedjica$) for collaborative source separation. Additionally, to address the privacy composition for multi-round algorithms, we analyzed our $\capedjica$ algorithm using R\'enyi differential privacy. To better measure for the privacy loss per iteration, we used the moments accountant method. We empirically compared the performance of the proposed algorithms with those of conventional, non-private and local algorithms on synthetic and real datasets. We varied privacy parameters and relevant dataset parameters to show that the proposed algorithms outperformed the conventional and local algorithms comfortably and matched the performance of the non-private algorithms for some parameter choices. In general, the proposed algorithms offered very good utility even for strong privacy guarantees -- indicating achievability of meaningful privacy even without loosing much utility. An interesting future work could be to extend the $\cape$ framework to $(\epsilon,0)$-differential privacy, perhaps using the Staircase Mechanism~\cite{geng2015} for differential privacy. Another possible direction is to extend $\cape$ to be employable in arbitrary tree-structured networks.

\bibliography{refs.bib}
\bibliographystyle{IEEEtran}

\clearpage

\appendices
\section{Problems with Conventional Decentralized DP Computations}\label{appendix:conventional-decentralized-dp}
As mentioned before, DP algorithms often introduce noise in the computation pipeline to induce the randomness. For additive noise mechanisms, the standard deviation of the noise is scaled to the \emph{sensitivity} of the computation~\cite{dwork2013algorithmic}. To illustrate, consider estimating the mean $f(\vect{x}) = \frac{1}{N} \sum_{n = 1}^N x_n$ of $N$ scalars $\vect{x} = [ x_1,\ldots, x_{N-1},\ x_N]^{\top}$ with each $x_i \in [0,1]$. The sensitivity of the function $f(\vect{x})$ is $\frac{1}{N}$. Therefore, for computing the $(\epsilon, \delta)$-DP estimate of the average $a = f(\vect{x})$, we can follow the Gaussian mechanism~\cite{dwork2006} to release $\hat{a}_\mrm{pool} = a + e_\mrm{pool}$, where $e_\mrm{pool} \sim \mathcal{N}\left(0, \tau_\mrm{pool}^2\right)$ and $\tau_\mrm{pool} = \frac{1}{N\epsilon}\sqrt{2\log \frac{1.25}{\delta}}$.

Suppose now that the $N$ samples are equally distributed among $S$ sites. That is, each site $s \in \{1,\ldots,S\}$ holds a disjoint dataset $\vect{x}_s$ of $N_s = N/S$ samples. An aggregator wishes to estimate and publish the mean of all the samples. For preserving privacy, the conventional DP approach is for each site to release (or send to the aggregator node) an $(\epsilon, \delta)$-DP estimate of the function $f(\vect{x}_s)$ as: $\hat{a}_s = f(\vect{x}_s) + e_s$, where $e_s \sim \mathcal{N}\left(0, \tau_s^2\right)$ and $\tau_s = \frac{1}{N_s\epsilon}\sqrt{2\log \frac{1.25}{\delta}} = \frac{S}{N\epsilon}\sqrt{2\log \frac{1.25}{\delta}}$. The aggregator can then compute the $(\epsilon, \delta)$-DP approximate average as $\hat{a}_\mrm{conv} = \frac{1}{S}\sum_{s=1}^S \hat{a}_s = \frac{1}{S}\sum_{s=1}^S a_s + \frac{1}{S}\sum_{s=1}^S e_s$. The variance of the estimator $\hat{a}_\mrm{conv}$ is $S \cdot \frac{\tau_s^2}{S^2} = \frac{\tau_s^2}{S} \triangleq \tau_\mrm{conv}^2$. We observe the ratio: $\frac{\tau_\mrm{pool}^2}{\tau_\mrm{conv}^2} = \frac{\tau_s^2 / S^2}{\tau_s^2 / S} = \frac{1}{S}$. That is, the decentralized DP averaging scheme will always result in a poorer performance than the pooled data case. We propose a protocol that improves the performance of such systems by assuming some reasonable resources.

\section{Additional Background Concepts}\label{appendix:background}
\noindent\textbf{R\'enyi Differential Privacy.} Analyzing the total privacy loss of a multi-shot algorithm, each stage of which is DP, is a challenging task. It has been shown~\cite{abadi2016, mironov2017} that the advanced composition theorem~\cite{dwork2013algorithmic} for $(\epsilon, \delta)$-differential privacy can be loose. The main reason is that one can formulate infinitely many $(\epsilon, \delta)$-DP algorithms for a given noise variance $\tau^2$. RDP offers a much simpler composition rule that is shown to be tight~\cite{mironov2017}. We review some properties of RDP~\cite{mironov2017}.

\begin{Prop}[From RDP to differential privacy~\cite{mironov2017}]\label{prop:rdp_to_dp}
If $\mathcal{A}$ is an $(\alpha, \epsilon_r)$-RDP mechanism, then it also satisfies $\left(\epsilon_r + \frac{\log \frac{1}{\delta_r}}{\alpha-1}, \delta_r\right)$-differential privacy for any $0 < \delta_r < 1$.
\end{Prop}

\begin{Prop}[Composition of RDP~\cite{mironov2017}]\label{prop:composition_rdp}
Let $\mathcal{A}: \mathbb{D} \mapsto \mathbb{T}_1$ be $(\alpha, \epsilon_{r1})$-RDP and $\mathcal{B}: \mathbb{T}_1 \times \mathbb{D} \mapsto \mathbb{T}_2$ be $(\alpha, \epsilon_{r2})$-RDP, then the mechanism defined as $(X, Y)$, where $X \sim \mathcal{A}(D)$ and $Y \sim \mathcal{B}(X, D)$, satisfies $(\alpha, \epsilon_{r1}+\epsilon_{r2})$-RDP.
\end{Prop}

\begin{Prop}[RDP and Gaussian Mechanism~\cite{mironov2017}]\label{prop:rdp_gauss_mech}
If $\mathcal{A}$ has $\mathcal{L}_2$ sensitivity 1, then the Gaussian mechanism $\matr{G}_\sigma \mathcal{A}(D) = \mathcal{A}(D) + E$, where $E\sim\mathcal{N}(0, \sigma^2)$ satisfies $(\alpha, \frac{\alpha}{2\sigma^2})$-RDP. Additionally, a composition of $J$ Gaussian mechanisms satisfies $(\alpha, \frac{\alpha J}{2\sigma^2})$-RDP.
\end{Prop}
\noindent The proofs of the Propositions \ref{prop:rdp_to_dp}, \ref{prop:composition_rdp} and \ref{prop:rdp_gauss_mech} are provided in~\cite{mironov2017}.

\noindent\textbf{The Moments Accountant.} The moments accountant method can be used to achieve a much smaller overall $\epsilon$ than the strong composition theorem~\cite{dwork2013algorithmic}. As mentioned before, na\"ively employing the additive nature of the privacy loss results in the worst case analysis, i.e., assumes that each iteration step exposes the worst privacy risk and this exaggerates the total privacy loss. However, in practice, the privacy loss is a random variable that depends on the dataset and is typically well-behaved (concentrated around its expected value). Let us consider the randomized mechanism $\mathcal{A}:  \mathbb{D} \mapsto \mathbb{T}$. For a particular outcome $o \in \mathbb{T}$ of the mechanism and neighboring datasets $D, D' \in \mathbb{D}$, the privacy loss random variable is defined~\cite{abadi2016} as
\begin{align}\label{eqn:priv_loss_rv}
Z &= \log\frac{\Pr[\mathcal{A}(D) = o]}{\Pr[\mathcal{A}(D') = o]} \mbox{ w.p. } \Pr[\mathcal{A}(D) = o].
\end{align}
Note that the basic idea of~\cite{abadi2016} for accounting for the total privacy loss is to compute the moment generating function (MGF) of $Z$ for each iteration, use composition to get the MGF of the complete algorithm and then use that to compute final privacy parameters (see Theorem 2 of~\cite{abadi2016}). The stepwise moment for any $t$ at iteration $j$ is defined~\cite{abadi2016} as
\begin{align}\label{eqn:mgf_itr}
\alpha_j(t) &= \sup_{D, D'} \log\mathbb{E}\left[\exp(tZ)\right].
\end{align}
If total number of iterations is $J^*$ then the overall moment is upper bounded as $\alpha(t) \leq \sum_{j=1}^{J^*} \alpha_j(t)$. Finally, for any given $\epsilon > 0$, the overall mechanism is $(\epsilon, \delta)$ DP for $\delta = \min_t \exp\left(\alpha(t) - t\epsilon\right)$.

\subsection{The Secure Aggregation Protocol}\label{appendix:secureagg}
For completeness, we describe the $\secureagg$ scheme~\cite{Bonawitz17} in more detail; this method generates correlated noise that fits our honest-but-curious setup. $\secureagg$ can securely compute sums of vectors. It has a constant number of rounds and a low communication overhead. It requires only one aggregator with limited trust and is robust to site-dropouts (up to a certain fraction). The aggregator only learns the inputs from the sites in aggregate and the overall scheme is proven to be secure against honest but curious adversaries. The scheme relies on Shamir's $t$-out-of-$n$ Secret Sharing~\cite{shamir1979} and Diffie-Hellman (DH) key agreement~\cite{diffie1976}. It is parameterized over a finite field $\mathbb{F}$ of size at least $l > 2^\lambda$ (where $\lambda$ is the security parameter of the scheme). The scheme has four stages: i) advertise keys, ii) share keys, iii) masked input collection and iv) unmasking. In the first stage, each site generates DH key pairs and sends the public keys to the aggregator. The aggregator waits for enough nodes and then broadcasts received public keys. In the second stage, each site $s_1$ generates a random seed $b_{s_1}$ and computes shared private key $k_{s_1, s_2}$ with site $s_2$. Site $s_1$ then computes $t$-out-of-$n$ shares of $b_{s_1}$ and $k_{s_1, s_2}$ and sends encrypted shares. The aggregator again waits for enough sites to send these encrypted shares. In the third stage, the aggregator forwards the received encrypted shares. Each site $s_1$ computes and sends masked input $e_{s_1}$. The aggregator waits for enough sites. In the last stage, the aggregator broadcasts a list of non-dropped-out sites. Each site $s_1$ sends shares of $b_{s_1}$ for these sites and the public keys $k_{s_1}^{PK}$ for dropped-out sites. Finally, the aggregator reconstructs the sum $\sum_s e_s$ and broadcasts to sites.

\section{Proof of Lemma~\ref{lemma:cape}}\label{appendix:cape_lemma}
\begin{proof}
We prove the lemma according to~\cite{imtiaz2018}. Recall that in the pooled data scenario, the sensitivity of the function $f(\vect{x})$ is $\frac{1}{N}$, where $\vect{x} = \left[\vect{x}_1,\ldots, \vect{x}_S\right]$. Therefore, to approximate $f(\vect{x})$ satisfying $(\epsilon, \delta)$ differential privacy, we need to have additive Gaussian noise standard deviation at least $\tau_\mrm{pool} = \frac{1}{N\epsilon}\sqrt{2\log\frac{1.25}{\delta}}$. Next, consider the distributed data setting (as in Section~\ref{sec:problem_formulation}) with local noise standard deviation given by $\tau_s = \frac{1}{N_s\epsilon}\sqrt{2\log\frac{1.25}{\delta}} = \frac{S}{N\epsilon}\sqrt{2\log\frac{1.25}{\delta}} = \tau$. We observe $\tau_\mrm{pool} = \frac{\tau_s}{S} \implies \tau_\mrm{pool}^2 = \frac{\tau^2}{S^2}$. We will now show that the $\cape$ algorithm will yield the same noise variance of the estimator at the aggregator. Recall that at the aggregator we compute $a_\mrm{cape} = \frac{1}{S} \sum_{s=1}^S \hat{a}_s = \frac{1}{N} \sum_{n=1}^{N} x_n + \frac{1}{S} \sum_{s=1}^S g_s$. The variance of the estimator $a_\mrm{cape}$ is: $\tau_\mrm{cape}^2 \triangleq S\cdot \frac{\tau_g^2}{S^2} = \frac{\tau_g^2}{S} = \frac{\tau^2}{S^2}$, which is exactly the same as the pooled data scenario. Therefore, the $\cape$ algorithm allows us to achieve the same additive noise variance as the pooled data scenario in the symmetric setting ($N_s = \frac{N}{S}$ and $\tau_s^2 = \tau^2\ \forall s \in [S]$), while satisfying $(\epsilon, \delta)$ differential privacy at the sites and for the final output from the aggregator, where $(\epsilon, \delta)$ satisfy the relation $\delta = 2\frac{\sigma_z}{\epsilon - \mu_z}\phi\left(\frac{\epsilon - \mu_z}{\sigma_z}\right)$.
\end{proof}

\section{Performance Improvement of $\cape$}\label{appendix:perf_gain}
\begin{figure}[ht]
  \centering
  \includegraphics[width=0.9\columnwidth]{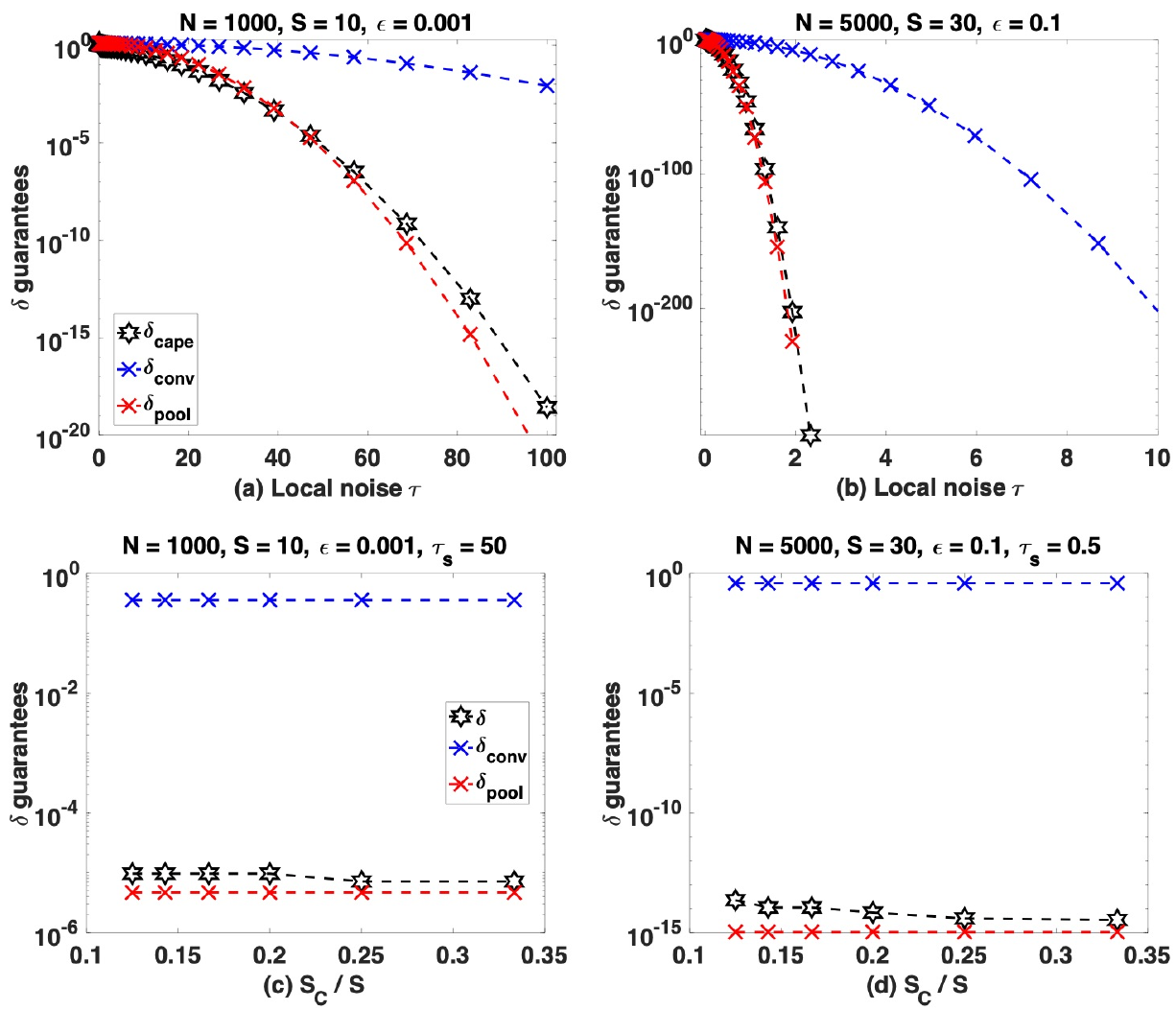}\\
  \vspace{-0.0in}
  \caption{Variation of $\delta_\mrm{cape}$, $\delta_\mrm{conv}$ and $\delta_\mrm{pool}$ with (a) -- (b) $\tau_s$ and (c) -- (d) $\frac{S_C}{S}$ for different values of $S$ and $\epsilon$}
  \label{fig:eff_delta}
\end{figure}
\begin{proof}[Proof of Proposition~\ref{prop:gain}]
The local noise variances are $\{\tau_s^2\}$ for $s \in [S]$. In the conventional decentralized DP scheme, we compute the following at the aggregator: $a_\mrm{conv} = \frac{1}{S}\sum_{s=1}^S a_s + \frac{1}{S}\sum_{s=1}^S e_s$. The variance of the estimator is: $\tau_\mrm{conv}^2 = \sum_{s=1}^s \frac{\tau_s^2}{S^2} = \frac{1}{S^2} \sum_{s=1}^s \tau_s^2$. In the $\cape$ scheme, we compute the following quantity at the aggregator: $a_\mrm{cape} = \frac{1}{S}\sum_{s=1}^S a_s + \frac{1}{S}\sum_{s=1}^S e_s + \frac{1}{S}\sum_{s=1}^S g_s$. The variance of the estimator is: $\tau_\mrm{cape}^2 = \sum_{s=1}^s \frac{\tau_g^2}{S^2} = \frac{1}{S^3} \sum_{s=1}^s \tau_s^2$. Therefore, the $\cape$ scheme provides a reduction $G = \frac{\tau_\mrm{conv}^2}{\tau_\mrm{cape}^2} = S$ in noise variance over conventional decentralized DP approach in the symmetric setting ($N_s = \frac{N}{S}$ and $\tau_s^2 = \tau^2\ \forall s \in [S]$), which completes the proof.
\end{proof}

\section{Proof of Proposition~\ref{prop:low_sensitivity}}\label{appendix:low_sensitivity}
First, we review some definitions and lemmas~\cite[Proposition C.2]{majorization} necessary for the proof. 
\begin{Def}[Majorization]
Consider two vectors $\vect{a} \in \mathbb{R}^S$ and $\vect{b} \in \mathbb{R}^S$ with non-increasing entries (i.e., $a_i \geq a_j$ and $b_i \geq b_j$ for $i < j$). Then $\vect{a}$ is majorized by $\vect{b}$, denoted $\vect{a} \prec \vect{b}$, if and only if the following holds:
\begin{align*}
\sum_{s=1}^S a_s &= \sum_{s=1}^S b_s \mbox{ and } \sum_{s=1}^J a_s \leq \sum_{s=1}^J b_s\ \forall J \in [S].
\end{align*}
\end{Def}

\noindent Consider $\vect{n}_\mrm{sym} \triangleq \frac{N}{S}[1,\ldots,1] \in \mathbb{R}^S$ for some positive $N$. Then any vector $\vect{n} = [N_1,\ldots,N_S] \in \mathbb{R}^S$ with non-increasing entries and $\sum_{s=1}^S |N_s| = N$ majorizes $\vect{n}_\mrm{sym}$, or $\vect{n}_\mrm{sym} \prec \vect{n}$.

\begin{Def}[Schur-convex functions]\label{def: SchurCVX}
The function $K: \mathbb{R}^{S} \mapsto \mathbb{R}$ is Schur-convex if for all $\vect{a} \prec \vect{b} \in \mathbb{R}^{S}$ we have $K(\vect{a}) \leq K(\vect{b})$.
\end{Def}

\begin{lemma}\label{lem: schurconx}
If $K$ is symmetric and convex, then $K$ is Schur-convex. The converse does not hold.
\end{lemma}

\begin{proof}[Proof of Proposition~\ref{prop:low_sensitivity}]
As the sites are computing the function $f$ with $\mathcal{L}_2$ sensitivity $\Delta(N)$, the local noise standard deviation for preserving privacy is proportional to $\Delta(N_s)$ by Gaussian mechanism~\cite{dwork2006}. It can be written as: $\tau_s = \Delta(N_s) C$, where $C$ is a constant for a given $(\epsilon, \delta)$ pair. Similarly, the noise standard deviation in the pooled data scenario can be written as: $\tau_\mrm{pool} = \Delta(N) C$. Now, the final noise variance at the aggregator for $\cape$ protocol is: $\tau^2_\mrm{cape} = \sum_{s=1}^S \frac{\tau_g^2}{S^2} = \frac{1}{S^3} \sum_{s=1}^S \Delta^2(N_s) C^2$. Observe: $H(\vect{n}) = \frac{\tau_\mrm{cape}^2}{\tau_\mrm{pool}^2} = \frac{\sum_{s=1}^S \Delta^2(N_s)}{S^3 \Delta^2(N)}$. As we want to achieve the same noise variance as the pooled-data scenario, we need $S^3 \Delta^2(N) = \sum_{s=1}^S \Delta^2(N_s)$, which proves the case for general sensitivity function $\Delta(N)$. Now, if $\Delta^2(N)$ is convex then the by Lemma~\ref{lem: schurconx} (Supplement) the function $K(\vect{n}) = \sum_{s=1}^S \Delta^2(N_s)$ is Schur-convex. Thus, the minimum of $K(\vect{n})$ is obtained when $\vect{n} = \vect{n}_\mrm{sym} = \left[\frac{N}{S},\ldots,\frac{N}{S}\right]$. We observe: $K_{\min}(\vect{n}) = \sum_{s=1}^S \Delta^2\left(\frac{N}{S}\right) = S \cdot \Delta^2\left(\frac{N}{S}\right)$. Therefore, for convex $\Delta(N)$, we have $H(\vect{n}) = 1$ if $\Delta(\frac{N}{S}) = S \Delta(N)$. 
\end{proof}

\section{Empirical Comparison of $\delta$ and $\delta_\mrm{conv}$}\label{appendix:eff_delta}
Recall that the $\cape$ protocol guarantees $(\epsilon, \delta)$-DP with $\delta = 2\frac{\sigma_z}{\epsilon - \mu_z}\phi\left(\frac{\epsilon - \mu_z}{\sigma_z}\right)$. We claim that this $\delta$ guarantee is much better than the $\delta$ guarantee in the conventional distributed DP scheme. As $\delta_\mrm{cape}$ is an implicit function of $S,\ S_C$ and $\tau_s^2$, we experimentally compare $\delta_\mrm{cape}$ with $\delta_\mrm{conv}$ and $\delta_\mrm{pool}$, where $\delta_\mrm{conv}$ and $\delta_\mrm{pool}$ are the smallest $\delta$ guarantees we can afford in the conventional distributed DP scheme and the pooled-data scenario, respectively, to achieve the same noise variance as the pooled-data scenario for a given $\epsilon$. Additionally, we are interested in how the $\delta_\mrm{cape}$, $\delta_\mrm{conv}$ and $\delta_\mrm{pool}$ vary with weaker collusion assumption (i.e., fewer colluding sites). To that end, we first plot $\delta_\mrm{cape}$, $\delta_\mrm{conv}$ and $\delta_\mrm{pool}$ against different $\tau_s$ values for $S_C = \ceil*{\frac{S}{3}} - 1$ and different combinations of $\epsilon$ and $S$ in Figure \ref{fig:eff_delta}(a)-(b). Next, we vary the fraction $\frac{S_C}{S}$ and plot the resulting $\delta_\mrm{cape}$, $\delta_\mrm{conv}$ and $\delta_\mrm{pool}$ for different combinations of $\epsilon$, $S$ and $\tau_s$ in Figure \ref{fig:eff_delta}(c)-(d). We observe from the figures that $\delta_\mrm{cape}$ is always smaller than $\delta_\mrm{conv}$ and smaller than the $\delta_\mrm{pool}$ for some $\tau$ values. That is, we are ensuring a much better privacy guarantee by employing the $\cape$ scheme over the conventional approach for achieving the same noise level at the aggregator output as the pooled data scenario. 

\section{Analysis of $\capedjica$ with the Moments Accountant Method}\label{appendix:moments-accountant}
We presented some preliminaries regarding moments accountant method in Appendix~\ref{appendix:background}. We now employ the framework to our $\capedjica$ algorithm and find the best $\epsilon$ for a given $\delta$. For a Gaussian mechanism $\matr{G}_\sigma \mathcal{A}(D) = \mathcal{A}(D) + E$, where $E\sim\mathcal{N}(0, \sigma^2)$, the privacy loss random variable defined in \eqref{eqn:priv_loss_rv} can be written as $Z = \log\frac{\exp\left(-\frac{\left(o - f_D\right)^2}{2\sigma^2}\right)}{\exp\left(-\frac{\left(o - f_D'\right)^2}{2\sigma^2}\right)} = \frac{2o(f_D - f_D') - (f_D^2 - {f_D'}^2)}{2\sigma^2}$. Here, the random variable $o$ is Gaussian with $o \sim \mathcal{N}(f_D, \sigma^2)$. Therefore, it can be shown that the random variable $Z$ is Gaussian: $Z \sim \mathcal{N}\left(\frac{(f_D - f_D')^2}{2\sigma^2}, \frac{(f_D - f_D')^2}{\sigma^2}\right)$. Using the moment generating function of generalized Gaussian, we have $\mathbb{E}[\exp(tZ)] = \exp\left(\frac{(f_D - f_D')^2}{2\sigma^2}\ t + \frac{1}{2}\frac{(f_D - f_D')^2}{\sigma^2}\ t^2\right) = \exp\left(\frac{(f_D - f_D')^2}{2\sigma^2}(t + t^2)\right)$. If the $\mathcal{L}_2$ sensitivity of the function $\mathcal{A}(D)$ is $\Delta$ then $\alpha_j(t) = \sup_{D,D'} \log \mathbb{E}[\exp(tZ)] = \sup_{D,D'} \frac{(f_D - f_D')^2}{2\sigma^2}(t + t^2) = \frac{\Delta^2}{2\sigma^2}(t + t^2)$. We can compute the upper bound of the overall moment: $\alpha(t) \leq \sum_{j=1}^{J^*} \alpha_j(t) = \frac{J^*\Delta^2}{2\sigma^2}(t + t^2)$. Now, for any given $\epsilon > 0$, we have $\delta  = \min_t \exp\left(\alpha(t) - t\epsilon\right) = \min_t \exp\left(\frac{J^*\Delta^2}{2\sigma^2}(t + t^2) - t\epsilon\right)$. We compute the minimizing $t$ as $t_\mrm{opt} = \frac{1}{2}\left(\frac{2\epsilon\sigma^2}{J^* \Delta^2}-1\right)$. Using this, we find $\delta_\mrm{opt}= \exp\left(\frac{J^*\Delta^2}{2\sigma^2}(t_\mrm{opt} + t_\mrm{opt}^2) - t_\mrm{opt}\epsilon\right)$. As we are interested in finding the best $\epsilon$ for a given $\delta_\mrm{target}$, we rearrange the above equation to formulate a quadratic equation in terms of $\epsilon$ and then solve for $\epsilon$: $\frac{\sigma^2}{2J^*\Delta^2}\epsilon^2 - \epsilon + \frac{J^*\Delta^2}{8\sigma^2} + \log\delta_\mrm{target} = 0$. Now, for the $\capedjica$ algorithm, we release two noisy gradients: $\Delta_\matr{W}(j)$ and $\Delta_\vect{b}(j)$, at iteration $j$ with noise variances $\sigma^2_\matr{W}$ and $\sigma^2_\vect{b}$, respectively. Adjusting for this, we plot the total $\epsilon$ against the total iterations $J^*$ for the strong composition~\cite{dwork2010}, RDP and the moments accountant in Figure \ref{fig:moments_accountant}. We observe that the moments accountant and RDP methods provide a much smaller overall $\epsilon$ than the strong composition.

\begin{figure}[t]
  \centering
  \includegraphics[width=0.8\columnwidth]{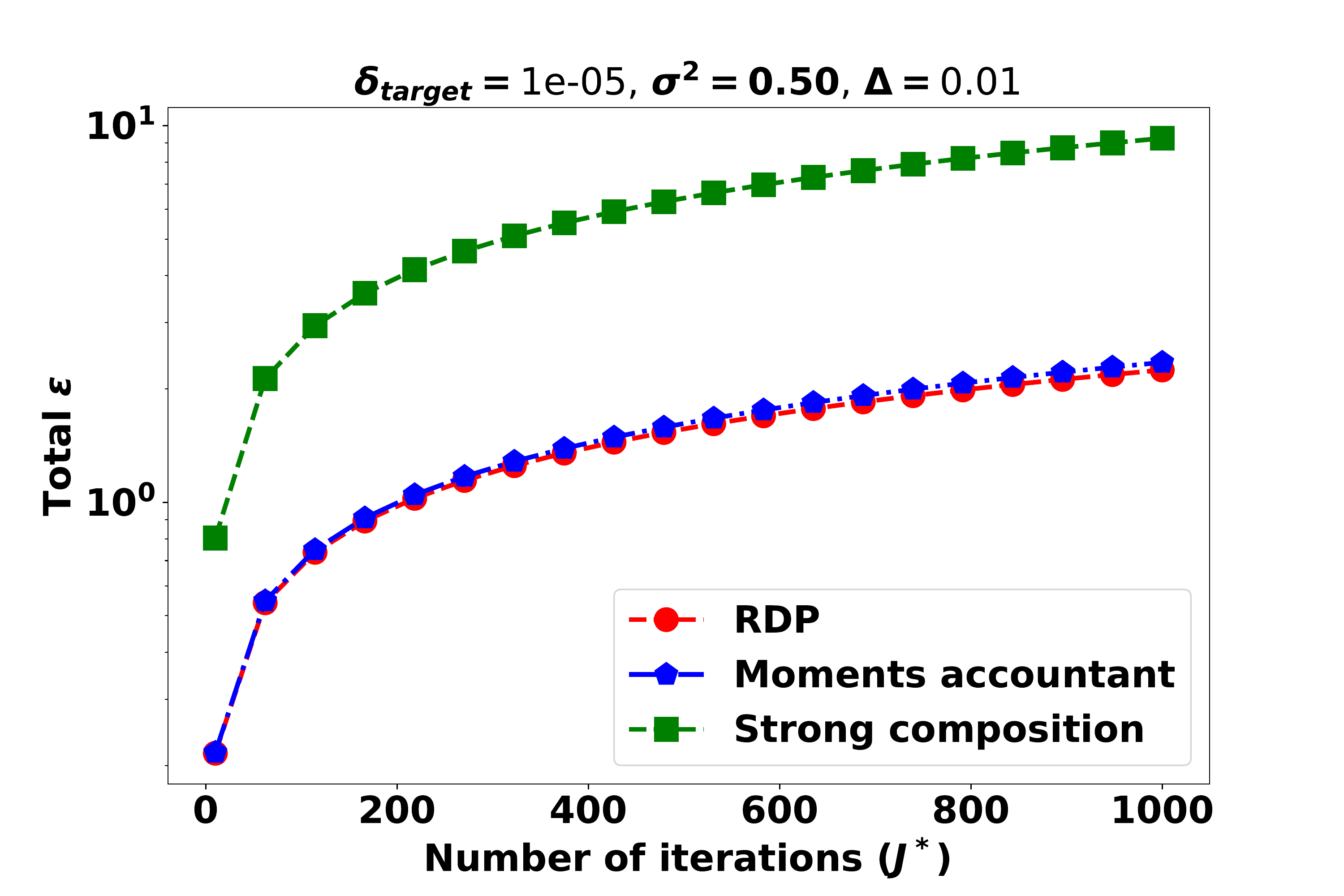}\\
  \vspace{-0.0in}
  \caption{Variation of total $\epsilon$ with number of iterations $J^*$: $\sigma_\matr{W}^2 = \sigma_\vect{b}^2 = 0.25$ and $\delta_\mrm{target} = 1e^{-5}$. Moments accountant and RDP methods provide a much smaller overall $\epsilon$ than the strong composition.}
  \vspace{-0.2in}
  \label{fig:moments_accountant}
\end{figure}

\section{Performance Variation with $\delta$}\label{appendix:perf-delta}
Recall that, the proposed $\capedjica$ algorithm guarantees $(\epsilon, \delta)$ differential privacy, where $(\epsilon, \delta)$ satisfy the relation $\delta = 2\frac{\sigma_z}{\epsilon - \mu_z}\phi\left(\frac{\epsilon - \mu_z}{\sigma_z}\right)$. In Figure \ref{fig:real_synth_delta}, we show the variation of $q^\mrm{NGI}$ with overall $\delta$ on synthetic and real data. Recall that $\delta$ is essentially the probability of failure of a DP algorithm. Therefore, we want $\delta$ to be small. However, a smaller $\delta$ also results in a larger noise variance, which affects the utility. From the figure, we can observe how the performance of the proposed $\capedjica$ algorithm varies with $\delta$ for a fixed $\epsilon_i = 0.5$. $\capedjica$ achieves very close utility to the non-private \djICA\ for both synthetic and real data. For both cases, the overall $\epsilon$ is also very small. However, we can opt for even smaller $\delta$ values at the cost of performance. 

\begin{figure}[t]
  \centering
  \includegraphics[width=0.48\textwidth]{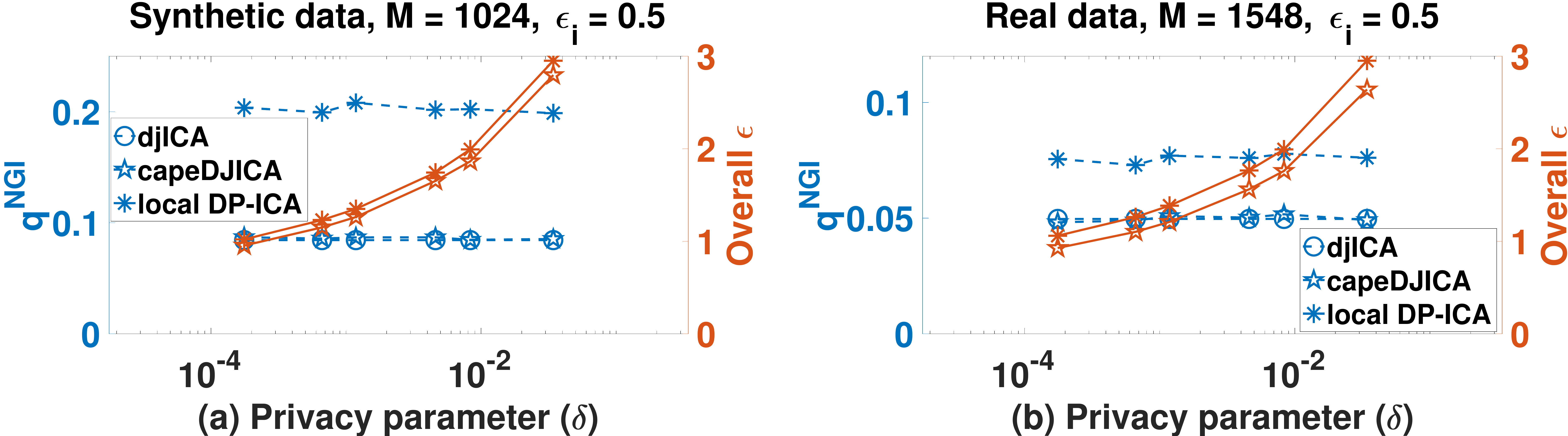}\\
  \vspace{-0.0in}
  \caption{Variation of $q^\mrm{NGI}$ and overall $\epsilon$ with privacy parameter $\delta$: (a) synthetic and (b) real fMRI data. Fixed parameters: $S = 4$, $\epsilon_i = 0.5$. $\capedjica$ achieves very close utility to the non-private \djICA\ with small overall $\epsilon$.}
  \vspace{-0.2in}
  \label{fig:real_synth_delta}
\end{figure}

\section{Communication Overhead for $\cape$} \label{appendix:cape_comm}
The conventional $D$-dimensional averaging needs only one message from each site, thus $SD$ or $\Theta(SD)$ is the communication complexity. Our $\cape$ scheme employs the $\secureagg$ protocol to compute the zero-sum noise. The $\secureagg$ protocol~\cite{Bonawitz17} entails an $O(S + D)$ overhead for each site and $O(S^2 + SD)$ for the server/aggregator. The rest of our scheme requires $\Theta(D)$ and $\Theta(SD)$ communication overheads for the sites and  the aggregator, respectively. On the other hand, the scheme proposed in~\cite{Mikko17} has a communication cost proportional to $(S+1)DM$ or $\Theta(SDM)$, where $M$ is the number of compute nodes. Goryczka et al.~\cite{SlawomirMPC17} compared several secret sharing, homomorphic encryption and perturbation-based secure sum aggregation and showed their communication complexities. Except for the secret sharing approach (which requires $O(S^2)$ overhead), the other approaches are $O(S)$ in communication complexity. A comparison of communication overhead for different algorithms are shown in Table \ref{tbl:commn_cost}.

\begin{table}[t]
\caption{Comparison of communication overhead}
\label{tbl:commn_cost}
\centering
\begin{tabular}{lll}
 \toprule
 \textbf{Algorithm} & \textbf{Site} & \textbf{Aggregator}\\
 \midrule
 $\cape$ & $O(S + D)$ & $O(S^2 + SD)$ \\
 Heikkil\"{a} et al.~\cite{Mikko17} & $\Theta(DM)$ & $\Theta(SDM)$\\
 Bonawitz et al.~\cite{Bonawitz17} & $O(S + D)$ & $O(S^2 + SD)$\\
 \bottomrule
\end{tabular}
\end{table}


\section{Decentralized Differentially Private PCA}\label{appendix:capepca}
As mentioned before, neuroimaging data samples are generally very high dimensional. We therefore use the recently proposed~\cite{imtiaz2018} DP decentralized PCA algorithm ($\capepca$) as an efficient and privacy-preserving dimention-reduction step of our proposed $\capedjica$ algorithm. 
We have shown a slightly modified version of the original $\capepca$ algorithm in Algorithm \ref{alg:dist_dpca} to match the robust $\cape$ scheme~\cite{imtiaz2019} we proposed in Section \ref{sec:cape}.

\begin{algorithm}[t] 
        \caption{Improved Decentralized Differentially Private PCA ($\capepca$)~\cite{imtiaz2018} \label{alg:dist_dpca}}
        \begin{algorithmic}[1]
    \Require Data matrix $\matr{X}_s \in \mathbb{R}^{D\times N_s}$ with $\|\vect{x}_{s,n}\|_2 \leq 1$, local noise variances $\{\tau_s^2\}$ for $s \in [S]$; reduced dimension $R$
    \For{$s = 1, 2, \ldots, S$} \Comment{at the local sites}
    	    \State Generate $\matr{E}_s \in \mathbb{R}^{D \times D}$ using Algorithm \ref{alg:zero-sum-noise-generation} (element-wise)
        \State Compute $\matr{C}_s \gets \frac{1}{N_s} \matr{X}_s \matr{X}_s^\top$
        \State Generate $D \times D$ symmetric $\matr{G}_s$, where $\{\left[\matr{G}_s\right]_{ij}: i \in [D], j \leq i\}$ are drawn i.i.d. $\sim \mathcal{N}(0, \tau_g^2 = \frac{1}{S} \tau_s^2)$, $[\matr{G}_s]_{ij} = [\matr{G}_s]_{ji}$
        \State Compute and send: $\hat{\matr{C}}_s \gets \matr{C}_s + \matr{E}_s + \matr{G}_s$
    \EndFor 
    \State Compute $\hat{\matr{C}} \gets \frac{1}{S}\sum_{s=1}^S \hat{\matr{C}}_s$ \Comment{at the aggregator}
    \State Perform SVD: $\hat{\matr{C}} = \matr{V} \matr{\Lambda} \matr{V}^\top$
    \State Release / send to sites: $\matr{V}_R$\\
    \Return $\matr{V}_R$
    \end{algorithmic}
\end{algorithm}

\section{Improved Differentially Private Decentralized Joint ICA ($\capedjica$) Algorithm}\label{appendix:capedjica}
\begin{algorithm*}[t]
\caption{Improved Differentially Private Decentralized Joint ICA ($\capedjica$) }\label{alg:capedjica}
\begin{algorithmic}[1]
\Require data $\{ \matr{X}_s^r \in \mathbb{R}^{R \times N_s} : s \in [S]\}$, tolerance level $t = 10^{-6}$, maximum iterations $J$, $\norm{\Delta_{\matr{W}}(0)}_2^2 = t$, initial learning rate $\rho = 0.015/\log(R)$, local noise standard deviations $\{\tau_G^s, \tau_h^s\}$, gradient bounds $\{B_G, B_h\}$
\State Initialize $j = 0$, $\matr{W} \in \mathbb{R}^{R \times R}$ \Comment{for example, $\matr{W} = \matr{I}$}
\While{$j < J$, $\norm{\Delta_{\matr{W}}(j)}_2^2 \geq t$}
        \ForAll{sites $s = 1, 2, \ldots, S$}
        		\State Generate $\matr{E}_s^G \in \mathbb{R}^{R\times R}$ and $\vect{e}_s^h \in \mathbb{R}^R$ using Algorithm \ref{alg:zero-sum-noise-generation} (element-wise)
	   		    \State Generate $\matr{K}_s^G \in \mathbb{R}^{R\times R}$ and $\vect{k}_s^h \in \mathbb{R}^R$, as \hyperlink{generate_K}{described} in the text
              \State Compute $\matr{Z}_s(j) = \matr{W}(j-1) \matr{X}_s + \vect{b}(j-1) \vect{1}^\top$ \label{eq:local_source_update}
                \State Compute $\hat{\matr{Y}}_s(j) = \matr{1} - 2g(\matr{Z}_s(j))$ \label{eq:entropy_transform}
                \State Compute $\matr{G}_s(j) = \frac{1}{N_s} \sum_{n=1}^{N_s} \matr{G}_{s,n}(j)$, where $\matr{G}_{s,n}(j) = \frac{\left(\matr{I} + \hat{\vect{y}}_{s,n}\vect{z}_{s,n}^\top\right)\matr{W}(j-1)}{\max\left(1, \frac{1}{B_G}\big\| \left(\matr{I} + \hat{\vect{y}}_{s,n}\vect{z}_{s,n}^\top\right)\matr{W}(j-1) \big\|_F\right)}$ 
                \State Compute $\vect{h}_s(j) = \frac{1}{N_s} \sum_{n=1}^{N_s} \vect{h}_{s, n}(j)$, where $\vect{h}_{s, n}(j) = \frac{\hat{\vect{y}}_{s,n}}{\max\left(1, \frac{1}{B_h} \big\| \hat{\vect{y}}_{s,n} \big\|_2\right)}$ \label{eq:local_grad_step}
        
        \State Compute $\hat{\matr{G}}_s(j) = \matr{G}_s(j) + \matr{E}_s^G + \matr{K}_s^G$
        \State Compute $\hat{\vect{h}}_s(j) = \vect{h}_s(j) + \vect{e}_s^h + \vect{k}_s^h$
                \State Send $\hat{\matr{G}}_s(j)$ and $\hat{\vect{h}}_s(j)$ to the aggregator
        \EndFor
        \State Compute $\Delta_\matr{W}(j) = \rho \frac{1}{S} \sum^{S}_{s=1} \hat{\matr{G}}_s(j)$ \label{alg:dp_djica:gradW}\Comment{at the aggregator, update global variables}
    \State Compute $\Delta_\vect{b}(j) = \rho \frac{1}{S} \sum^{S}_{s=1} \hat{\vect{h}}_s(j)$ \label{alg:dp_djica:gradb}
        \State Compute $\matr{W}(j) = \matr{W}(j-1) + \Delta_\matr{W}(j)$
        \State Compute $\vect{b}(j) = \vect{b}(j-1)  + \Delta_\vect{b}(j)$ \label{eq:agg_bias}
        \State Check upper bound and perform learning rate adjustment (if needed)
        \State Send global $\matr{W}(j)$ and $\vect{b}(j)$ back to each site
        \State $j \gets j + 1$
\EndWhile\\
\Return The current $\matr{W}(j)$
\end{algorithmic}
\end{algorithm*}

\subsection{Privacy Analysis}\label{appendix:djica-naive-privacy}
For completeness, we present a theorem that provides the privacy guarantee of the $\capedjica$ algorithm using the conventional composition theorem~\cite{dwork2013algorithmic}. For a tighter privacy bound, please see Section~\ref{sec:djica-renyi-dp} and Section~\ref{sec:djica-moments-accountant}.

\begin{theorem}[Privacy of $\capedjica$ Algorithm]\label{thm:capedjica}
Consider Algorithm \ref{alg:capedjica} in the decentralized data setting of Section \ref{sec:problem_formulation} with $N_s = \frac{N}{S}$, $\tau_G^s = \tau_G$ and $\tau_h^s = \tau_h$ for all sites $s \in [S]$. Suppose that at most $S_C = \ceil*{\frac{S}{3}} - 1$ sites can collude after execution and the required number of iterations is $J^*$. Then Algorithm \ref{alg:capedjica} computes an $(2J^*\epsilon, 2J^*\delta)$-DP approximation to the optimal unmixing matrix $\matr{W}^*$, where $(\epsilon,\delta)$ satisfy the relation $\delta = 2\frac{\sigma_z}{\epsilon - \mu_z}\phi\left(\frac{\epsilon - \mu_z}{\sigma_z}\right)$ and $(\mu_z, \sigma_z)$ are given by \eqref{eq:CAPEmu} and \eqref{eq:CAPEsigma}, respectively.
\end{theorem}
\begin{proof}The proof of Theorem \ref{thm:capedjica} follows from using the Gaussian mechanism~\cite{dwork2006}, the decentralized Stochastic Gradient Descent algorithm~\cite{abadi2016, bassily2014, anand2015}, the $\mathcal{L}_2$ sensitivities of the functions $f(\matr{X}_s) = \matr{G}_s$ and $f(\matr{X}_s) = \vect{h}_s$ and the privacy of the $\cape$ scheme~\cite{imtiaz2019}. We recall that the data samples in each site are disjoint. By the $\cape$ scheme (see Theorem \ref{thm:cape}), each iteration round is $(2\epsilon, 2\delta)$-DP. If the required total number of iterations is $J^*$ then by composition theorem of differential privacy~\cite{dwork2013algorithmic}, the $\capedjica$ algorithm satisfies $(2J^*\epsilon, 2J^*\delta)$-differential privacy, where $(\epsilon,\delta)$ satisfy the relation $\delta = 2\frac{\sigma_z}{\epsilon - \mu_z}\phi\left(\frac{\epsilon - \mu_z}{\sigma_z}\right)$.
\end{proof}

\end{document}